\documentclass[a4paper,11pt,leqno]{article}
\usepackage[T1]{fontenc}
\usepackage{lmodern}
\usepackage{amsmath,amsfonts,bm, amsthm}
\usepackage{url}
\usepackage{float}

\usepackage[ruled,vlined]{algorithm2e}
\usepackage{algpseudocode}
\usepackage[backref]{smacros}

\newcommand{\ie}{\textit{i.e. }}
\def\vp{{\bm{p}}}
\def\vq{{\bm{q}}}

\def\vx{{\bm{x}}}
\def\vy{{\bm{y}}}

\def\gD{{\mathcal{D}}}

\def\gH{{\mathcal{H}}}

\def\gO{{\mathcal{O}}}

\def\gU{{\mathcal{U}}}

\def\gX{{\mathcal{X}}}

\def\sD{{\mathbb{D}}}
\def\sP{{\mathbb{P}}}
\def\sQ{{\mathbb{Q}}}
\def\sR{{\mathbb{R}}}
\def\sS{{\mathbb{S}}}

\def\sU{{\mathbb{U}}}

\def\sX{{\mathbb{X}}}

\def\mA{{\bm{A}}}

\title{Sublinear Sketches for Approximate Nearest Neighbor and Kernel Density Estimation}

\author{
Ved Danait\thanks{Department of Computer Science \& Engineering, Indian Institute of Technology Bombay, Mumbai, India.\\ Email: \textcolor{cyan}{veddanait@cse.iitb.ac.in}}
\quad
Srijan Das\thanks{Department of Computer Science \& Engineering, Indian Institute of Technology Bombay, Mumbai, India.\\ Email: \textcolor{cyan}{srijandas@cse.iitb.ac.in}}
\quad
Sujoy Bhore\thanks{Department of Computer Science \& Engineering and Centre for Machine Intelligence \& Data Science (CMINDS),\\ Indian Institute of Technology Bombay, Mumbai, India.\\ Email: \textcolor{cyan}{sujoy@cse.iitb.ac.in}}
}
\date{}

\begin{document}

\maketitle

\begin{abstract}
Approximate Nearest Neighbor (\textsf{ANN}) search and Approximate Kernel Density Estimation (\textsf{A-KDE}) are fundamental problems at the core of modern machine learning, with broad applications in data analysis, information systems, and large-scale decision making. In massive and dynamic data streams, a central challenge is to design compact \emph{sketches} that preserve essential structural properties of the data while enabling efficient queries.

In this work, we develop new \emph{sketching algorithms} that achieve \emph{sublinear} space and query time guarantees for both \textsf{ANN} and \textsf{A-KDE} for a dynamic stream of data. For \textsf{ANN} in the \emph{streaming} model, under natural assumptions, we design a \emph{sublinear sketch} that requires only $\mathcal{O}(n^{(1-\eta)(1+\rho)})$ memory by storing only a sublinear ($n^{-\eta}$) fraction of the total inputs, where $\rho$ is a parameter of the \textsc{LSH} family, and $0<\eta<1$. Our method supports sublinear query time, batch queries, and extends to the more general \emph{Turnstile} model.
While earlier works have focused on Exact \textsf{NN}, this is the first result on \textsf{ANN} that achieves near-optimal trade-offs between memory size and approximation error. 

Next, for \textsf{A-KDE} in the \emph{Sliding-Window} model, we propose a sketch of size $\mathcal{O}\left(LW \cdot \frac{1}{\sqrt{1+\epsilon} - 1} \log^2 N\right)$, where $L$ is the number of sketch rows, 
$W$ is the \textsc{LSH} range, $N$ is the window size, and $\epsilon$ is the approximation error. This, to the best of our knowledge, is the first theoretical \emph{sublinear sketch} guarantee 
for \textsf{A-KDE} in the \emph{Sliding-Window} model.

We complement our theoretical results with experiments on various real-world datasets, which show that the proposed sketches are lightweight and achieve consistently low error in practice. 
\end{abstract}

%keywords: Approximate Nearest Neighbor, Approximate Kernel Density Estimation, Streaming Algorithms, Sublinear Sketches. 

\newpage
\tableofcontents
\newpage
\section{Introduction}
Modern machine learning and data analysis often require answering similarity and density queries over massive and evolving datasets. As data grows in scale and arrives in dynamic streams, it becomes infeasible to store or process everything explicitly, and the central challenge is to design compact sketches that preserve essential structure while supporting efficient queries. This challenge has been highlighted in seminal works on sublinear algorithms, streaming, and high-dimensional data processing \cite{alon1996space, indyk1998approximate, gionis1999similarity, muthukrishnan2005data, cormode2005improved}. Two central problems in this setting are Approximate Nearest Neighbor (\textsf{ANN}) search and Approximate Kernel Density Estimation (\textsf{A-KDE}).

\paragraph{Approximate Nearest Neighbor (\textsf{ANN}):} 
The Approximate Nearest Neighbor (\textsf{ANN}) problem formalizes similarity search in high-dimensional spaces. In its standard $(c,r)$-formulation, given a point set $\sP$ in a metric space $(\gX, \gD)$, the task is to preprocess $\sP$ so that, for any query $\vq \in \gX$, if the distance to its nearest neighbor in $\sP$ is at most $r$, then with probability at least $1 - \delta$ the algorithm returns a point within distance $cr$ of $\vq$. This formulation corresponds to the modern version of the $(r_1, r_2)$-PLEB problem introduced by \cite{indyk1998approximate}, with a relaxed variant due to \cite{har2012approximate} that permits null or arbitrary answers if no point lies within distance $r_1$. 
In~\cite{har2012approximate}, Har-Peled et al. presented fully dynamic solutions to \textsf{ANN}, though such methods require at least linear space and are not suited for the streaming model. In parallel, locality-sensitive hashing–based frameworks (e.g., \cite{panigrahy2005entropy, andoni2008near, andoni2014beyond, andoni2015optimal, andoni2017optimal, ahle2017optimal}) have driven much of the progress on optimizing space–time trade-offs in the primarily static and dynamic settings rather than streaming. 

%In a seminal work of,~\cite{har2012approximate} obtained fully dynamic solutions to \textsf{ANN}, though such methods require at least linear space and are not designed for the streaming model. Subsequent advances, particularly those based on locality-sensitive hashing (e.g., \cite{panigrahy2005entropy, andoni2008near, andoni2014beyond, andoni2015optimal, andoni2017optimal, ahle2017optimal}), have achieved strong trade-offs between space and query time, though again primarily in static or dynamic settings rather than streaming.

\paragraph{Approximate Kernel Density Estimation (\textsf{A-KDE}):} 
Kernel density estimation (\textsf{KDE}) is a classical non-parametric method for estimating probability distributions from data \cite{davis2011remarks, parzen1962estimation,  silverman2018density, scott2015multivariate}. Given a sequence of i.i.d. random variables $x_1, \dots, x_n \in \mathbb{R}^d$, the density at a query $x$ is estimated as
$
\hat p(x;\sigma) = \frac{1}{n}\sum_{i=1}^n K_\sigma(x - x_i),$
where $K_\sigma(\cdot)$ is a kernel function with bandwidth $\sigma$. While highly effective, exact \textsf{KDE} becomes computationally prohibitive for large-scale or streaming datasets, which motivated the approximate version (namely, \textsf{A-KDE}), where the goal is to return, for any query $q$, a $(1 + \epsilon)$ multiplicative approximation to the true density with probability at least $1-\delta$. \textsf{A-KDE} thus provides a principled framework for large-scale density estimation, balancing accuracy and efficiency. Notable advances include \textsc{Repeated Array of Counts Estimator} (in short, \textsc{RACE}) \cite{coleman2020sub}, which compresses high-dimensional vectors into compact counters; \textsf{TAKDE} \cite{wang2023takde}, evaluated experimentally in the sliding-window setting; and \textsf{KDE-Track} \cite{qahtan2016kde}, designed for spatiotemporal streams.

%large-scale retrieval problems
\paragraph{Streaming Applications:} Consider a personalized news agent or financial assistant powered by large language models: vast streams of articles or market updates arrive dynamically, yet the system must provide timely, personalized insights without storing or processing the entire corpus. Approximate nearest neighbor (\textsf{ANN}) search enables real-time matching of a user’s evolving interests to relevant news items or market updates, while approximate kernel density estimation (\textsf{A-KDE}) captures shifts in topical or market distributions to adapt recommendations. A similar challenge arises in large-scale image and video platforms, where streams of photos or frames arrive continuously. \textsf{ANN} supports fast similarity search for recommendation, moderation, or retrieval in these large-scale systems, while \textsf{A-KDE} tracks distributional changes such as emerging styles, anomalies, or trending categories.

Such scenarios are not limited to text or vision: related needs appear in personalization, anomaly detection, and monitoring of high-volume data streams \cite{alon1996space, indyk1998approximate, muthukrishnan2005data, cormode2005improved, jegou2011product, coleman2020sub, qahtan2016kde, wang2023takde}. Across these domains, storing or processing all data explicitly is infeasible, making compact sketches essential for balancing efficiency and accuracy in large-scale retrieval problems.

\begin{center}
    \emph{A central question is how to efficiently perform approximate nearest neighbor (\textsf{ANN}) search and approximate kernel density estimation (\textsf{A-KDE}) on massive, dynamically evolving data streams.}
\end{center}

In this work, we use three most commonly used models of streaming, namely, insertion-only, turnstile, and sliding window.

\paragraph{Insertion-Only:} In the insertion-only streaming model, data arrives sequentially and can only be appended to the dataset; deletions are not allowed. This model captures many practical scenarios, such as log analysis, clickstreams, and sensor readings, where storing all data explicitly is infeasible. The goal is to maintain a compact summary that supports approximate queries using sublinear memory. Classical sketches such as Count-Min \cite{cormode2005improved}, AMS \cite{alon1996space}, and their extensions to similarity search and density estimation \cite{indyk1998approximate, coleman2020sub} have demonstrated the effectiveness of insertion-only algorithms for both high-dimensional similarity search and approximate kernel density estimation.

%In many practical use cases, the scale of the data makes it infeasible to store all data points. The goal of this model is to allow ``insertion-only'' updates, while using sublinear memory. 

\paragraph{Turnstile:} The Turnstile model generalizes insertion-only streams by allowing both additions and deletions of data elements. This model is essential in settings where the dataset evolves dynamically or counts need to be adjusted, such as network traffic monitoring, dynamic graphs, and streaming recommendation updates. Maintaining sublinear sketches under Turnstile updates is more challenging, but prior work has shown that linear sketches, hash-based methods, and \textsc{RACE}-style counters can provide provable guarantees on query accuracy while supporting deletions efficiently \cite{indyk1998approximate, cormode2005improved, coleman2020sub}. For \textsf{ANN}, fully dynamic \textsc{LSH} and entropy-based methods have also been developed to handle updates in this model \cite{har2012approximate, andoni2017optimal}.

%This generalizes the vanilla model by allowing both insertions and deletions, while still maintaining sublinear memory. 

\paragraph{Sliding Window:} 
In many applications, only the most recent data is relevant. The sliding-window model maintains a succinct summary of the last $W$ updates, automatically expiring older elements. This model is particularly suitable for time-evolving datasets such as streaming video, sensor networks, or financial transactions, where queries should adapt to current trends. Exact computation of statistics like kernel density is often infeasible in this setting, motivating the development of approximate sketches. Techniques such as the \textsc{Exponential Histogram} (in short, \textsc{EH}) \cite{datar2002maintaining}, combined with \textsc{RACE} \cite{coleman2020sub} or other linear sketches, allow efficient approximation of both \textsf{ANN} and \textsf{KDE} over sliding windows \cite{wang2023takde, qahtan2016kde}. These methods balance memory efficiency, update speed, and approximation guarantees, making them practical for large-scale streaming environments.

\subsection{Problem Definitions}\label{sec:formal-def}

\begin{problem}[\emph{Streaming $(c,r)$-Approximate Near Neighbor (\textsf{ANN})}]\label{prob:ann}
Let $(\mathcal{X}, \mathcal{D})$ be a metric space, and let $\mathcal{P} \subseteq \mathcal{X}$ be a stream of at most $n$ points. Let $B(\mathbf{p}, r)$ denote a ball of radius $r$ around point $\mathbf{p}$. The goal is to maintain a data structure over the stream such that, for any query point $\mathbf{q} \in \mathcal{X}$:

\begin{itemize}
    \item If $\mathcal{D}_{\mathcal{P}}(\mathbf{q}) \le r$, then with probability at least $1 - \delta$ (for $0 < \delta \le 1$), the data structure returns some $\mathbf{p}' \in \mathcal{P} \cap B(\mathbf{q}, c r)$.
    \item The data structure stores only a sublinear fraction of the stream, i.e., $\mathcal{O}(n^{1-\eta})$ points (for $0 < \eta \le 1$), while supporting efficient updates and queries.
\end{itemize}

We refer to this task as the \emph{Streaming $(c,r)$-\textsf{ANN}} with failure probability $\delta$.
\end{problem}

\begin{problem}[\emph{Sliding-Window Approximate Kernel Density Estimation (\textsf{A-KDE})}]\label{prob:kde}
Let $\{\mathbf{x}_t\}_{t\ge 1}$ be a data stream, where each $\mathbf{x}_t$ is drawn from a (potentially time-varying) density $p_t(\mathbf{x})$. Let $N$ denote the window size, and let $\mathcal{T}_t = \{t-N+1, \dots, t\}$ denote the indices of points in the current window.

Given a query $\mathbf{x}$, the sliding-window \textsf{KDE} at time $t$ is defined as:
\[
\hat h(\mathbf{x}; \sigma_t) = \frac{1}{N} \sum_{j \in \mathcal{T}_t} K_{\sigma_t}(\mathbf{x} - \mathbf{x}_j),
\]
where $K_{\sigma_t}$ is a kernel with bandwidth $\sigma_t$.

The goal is to maintain a compact sketch that supports \emph{\textsf{A-KDE}} over the sliding window, enabling efficient updates and queries.
\end{problem}

\subsection{Our Contributions}

This work makes progress on two central problems in large-scale data analysis: streaming Approximate Nearest Neighbor (\textsf{ANN}) search (Problem~\ref{prob:ann}) and Approximate Kernel Density Estimation (\textsf{A-KDE}) (Problem~\ref{prob:kde}). We design new sketching algorithms that provably achieve sublinear space while supporting efficient queries, and we validate their effectiveness through extensive experiments. Below, we highlight our key contributions.

\subsubsection{Our contributions to \textsf{ANN} (Problem~\ref{prob:ann})}\label{sec:cont_ann}  At first glance, maintaining even approximate solutions for \textsf{ANN} in the streaming model with \emph{sublinear sketches} appears rather hopeless: an adversary can force any algorithm to store nearly all the data by giving inputs from scaled multidimensional lattices. However, real-world data is far from adversarial and often follows natural distributional assumptions \cite{mou2017refined, coleman2019sub}. Leveraging this, we prove that under a Poisson point process model—a well-studied and practically relevant distribution—\textsf{ANN} in the streaming setting \emph{does} admit efficient sketching.

Our approach revisits the classical Motwani–Indyk framework \cite{indyk1998approximate}. We show that, under Poisson distributed inputs, it suffices to retain only a sublinear fraction of the stream, namely $\gO(n^{1-\eta})$ points obtained by uniform sampling. This leads to a simple yet powerful sketching scheme with the following guarantees:
\begin{enumerate}
    \item \textbf{Streaming \textsf{ANN} with sublinear space.} Our sketch provides $(c,r)$-\textsf{ANN} guarantees while storing only a vanishing fraction of the input. 
    \item \noindent \textbf{Turnstile robustness.} We extend our guarantees to the Turnstile model, assuming only mild restrictions on adversarial deletions within any unit ball.  
    \item \noindent \textbf{Parallel batch queries.} Our scheme naturally supports batch queries, which can be executed in parallel to achieve significant speedups. 
\end{enumerate}

To our knowledge, this is the first work to obtain such guarantees for \textsf{ANN} in the streaming model under realistic assumptions. Importantly, the simplicity of our scheme makes it broadly applicable and easy to generalize.

\medskip\textbf{Empirical validation.} We complement our theoretical analysis with experiments on real-world datasets --- \texttt{sift1m} and \texttt{fashion-mnist}, as well as a custom synthetic dataset, \texttt{syn-32}, designed to emulate a Poisson point process. The results demonstrate that our sketches are lightweight, achieve consistently low error, and provide \emph{truly sublinear} space usage in high-$\epsilon$ regimes without compromising accuracy. In particular:  
\begin{enumerate}
    \item We show that for $\epsilon = 0.5$, we obtain sublinear sketches for all $\eta \geq 0.5$. More generally, for every sufficiently large $\epsilon$, there exists a threshold $\eta^*$ such that for all $\eta > \eta^*$, our scheme guarantees sublinear sketches without compromising on performance.  
    \item Our method outperforms the Johnson–Lindenstrauss (JL) baseline: beyond $\epsilon \approx 0.7$–$0.8$ on \texttt{sift1m}, and beyond $\epsilon \approx 0.9$ on \texttt{fashion-mnist}.  
    \item We demonstrate that, under fixed workloads, our method achieves higher query throughput than the \textsf{JL} baseline across diverse datasets and parameter settings. Moreover, with respect to RACE-CMS (\cite{coleman2019sub}), we show that our method is more accurate without requiring too much more memory, and faster in terms of the query complexity.
    %\textcolor{pink}{We demonstrate that, under fixed workloads, our method consistently delivers higher query throughput than the \textsf{JL} baseline across multiple datasets and over a wide range of parameter values.}
\end{enumerate}

\subsubsection{Our contributions to \textsf{A-KDE} (Problem~\ref{prob:kde})}

The \textsc{RACE} algorithm of \cite{coleman2020sub} provides an elegant sketch for \textsf{KDE} in dynamic data streams and naturally supports the Turnstile model, thanks to its ability to handle both insertions and deletions. However, \textsc{RACE} lacks the mechanism to manage temporal information explicitly, making it unsuitable for the \emph{sliding-window} model where data must expire once it falls outside the most recent $N$ updates. 

To address this challenge, we incorporate the classical \textsc{EH}\footnote{An exponential histogram can maintain aggregates over data streams with respect to the last $N$ data elements. For example, it can estimate up to a certain error the number of 1's seen within the last $N$ elements, assuming data is in the form of 0s and 1s.} result \cite{datar2002maintaining} into each \textsc{RACE} cell. \textsc{EH} is a powerful tool for maintaining aggregates over the most recent $N$ updates with provable accuracy guarantees, and here they enable us to count, with bounded error, how many elements in the active window hash to the same \textsc{LSH} bucket as the query $\vq$. This delicate combination allows us to design the \emph{first sketch} for the \textbf{\textsf{A-KDE} in the sliding-window model}, which explicitly handles expiration of old data while retaining the efficiency of \textsc{RACE}. Our construction does incur an extra $\log^2 N$ factor in space compared to plain \textsc{RACE}, but it uniquely enables sliding-window guarantees. We further extend this approach to handle \emph{batch updates}, where data arrives in mini-batches; here, the window consists of $N$ batches, and the \textsc{EH} is naturally adapted to this setting.  

\medskip\textbf{Empirical validation.} 
We evaluate our sliding-window \textsf{A-KDE} sketch on both synthetic and real-world datasets. We have used \texttt{News Headlines} and \texttt{ROSIS Hyperspectral Images} as real datasets. For synthetic data, we have conducted \emph{Monte Carlo} simulations using 50 different sets of $10000$ data points of dimension 200 sampled from $10$ different multivariate Gaussian distributions. We use a window size parameter to define the sliding window size. Our algorithm can also support batch updates. The experimental results highlight that:  
\begin{enumerate}
    \item The empirical relative error of \textsf{KDE} estimates is significantly smaller than the worst-case theoretical bound, even with a small number of rows in the sketch. This is observed for both synthetic and real-world datasets.
    \item Our sliding-window \textsf{A-KDE} achieves accuracy comparable to \textsc{RACE} \cite{coleman2020sub} on \texttt{News Headlines} and \texttt{ROSIS Hyperspectral Images} as well as the synthetic dataset.  
\end{enumerate}

%Due to space paucity, proofs of the Lemmas and Theorems marked with $(\star)$ are deferred to the Appendix. 

\section{Preliminaries}
\subsection{Locality Sensitive Hashing} 

Let $\sX$ be a metric space and $D:\sX\times \sX\rightarrow\sR$ be a distance metric.
\begin{definition} A family $\mathcal{H} = \{ h: \sX \to \sU \}$ is \textit{$(r_1, r_2, p_1, p_2)$-sensitive} for $(\sX, D)$ if for any $\vp, \vq \in \sX$, we have:
\begin{itemize}
    \item If $D(\vp, \vq) \leq r_1$, then $P_{h \in \mathcal{H}}[h(\vq) = h(\vp)] \geq p_1$.
    \item If $D(\vp, \vq) \geq r_2$, then $P_{h \in \mathcal{H}}[h(\vq) = h(\vp)] \leq p_2$.
\end{itemize}
\end{definition}

For a locality-sensitive family to be useful, it must satisfy the inequalities $\vp_1 > \vp_2$ and $r_1 < r_2$.
Two points $\vp$ and $\vq$ are said to collide, if $h(\vp)=h(\vq)$. We denote the collision probability by $k(\vx,\vy)$. Note that $k(.,.)$ is bounded and symmetric \ie $0\leq k(\vx,\vy)\leq 1,k(\vx,\vy)=k(\vy,\vx),$ and $k(\vx,\vx)=1$. It is known that if there exists a hash function $h(\vx)$ with $k(\vx,\vy)$ and range $[1,W]$, the same hash function can be independently concatenated $p$ times to obtain a new hash function $H(.)$ with collision probability $k^p(\vx,\vy)$ for any positive integer $p$. The range of the new hash function will be $[1,W^p]$.  
In particular, we use two such hash families in our analysis: (1) \textsc{SRP-LSH} (also known as, Angular \textsc{LSH}), described in \cite{charikar2002similarity} and (2) $p$-stable \textsc{LSH}, described in \cite{datar2004locality}.

\subsection{Approximate Nearest  Neighbor} 

In \cite{har2012approximate}, the authors give a scheme to solve the (c,r)-\textsf{ANN} problem with the following guarantees: 
\begin{theorem}[\cite{har2012approximate}]
    Suppose there is an $(r, cr, p_1, p_2)$-sensitive family $\mathcal{H}$ for $(\gX, \gD)$, where $p_1, p_2 \in (0,1)$ and let $\rho = \frac{\log(\frac{1}{p_1})}{\log(\frac{1}{p_2})}$. Then there exists a fully dynamic data structure for the $(c, r)$-Approximate Near Neighbor Problem over a set $\sP \subset \gX$ of at most $n$ points, such that:
    \begin{itemize}
        \item Each query requires at most $\gO\!\left(\frac{n^{\rho}}{p_{1}}\right)$
    distance computations and $\gO\!\left(\tfrac{n^{\rho}}{p_{1}} \cdot \log_{\frac{1}{p_{2}}} n\right)$
    evaluations of hash functions from $\mathcal{H}$. The same bounds hold for updates.  
        \item The data structure uses at most $\gO(\frac{n^{1+\rho}} {p_1})$ words of space, in addition to the space required to store $\sP$.
    \end{itemize}
    The failure probability of the data structure is at most $\frac{1}{3} + \frac{1}{e} < 1$.
\end{theorem}
Given an $(r, cr, p_1, p_2)$-sensitive family $\mathcal{H}$ of hash functions, the authors amplify the gap between the ``high'' probability $p_1$ and ``low'' probability $p_2$ by concatenating several functions. For a specified parameter $k$, they define a function family $\mathcal{G} = \{ g: \gX \to \gU^k \}$ such that
\[
g(\vp) = (h_{i_1}(\vp), h_{i_2}(\vp), \dots, h_{i_k}(\vp))
\]
where $h_i \in \mathcal{H}$ and $I = \{ i_1, \dots, i_k \} \subset \{1, \dots, |\mathcal{H}|\}$. For an integer $L$, they choose $L$ functions $g_1, \dots, g_L$ from $\mathcal{G}$ independently and uniformly at random. During preprocessing, they store a pointer to each $\vp \in \sP$ in the buckets $g_1(\vp), \dots, g_L(\vp)$. Since the total number of buckets may be large, they retain only the non-empty buckets by resorting to ``standard'' hashing.

\subsection{Repeated Array-of-Counts Estimator (\textsc{RACE})}

In \cite{coleman2020sub}, the authors propose \textsc{RACE}, an efficient sketching algorithm for kernel density estimation on high-dimensional streaming data. The \textsc{RACE} algorithm compresses a dataset $\sD$  into a 2-dimensional array $\mA$ of integer counters of size $L\times R$ where each row is an \textsc{Arrays of (locality-sensitive) Counts Estimator}(\textsc{ACE}) data structure\footnote{\textsc{ACE} is a swift and memory efficient algorithm used for unsupervised anomaly detection which does not require to store even a single data sample and can be dynamically updated.}\cite{luo2018arrays}. To add an element $\vx\in\sD$ we compute the hash of $\vx$ using $L$ independent \textsc{LSH} functions $h_1(\vx),h_2(\vx),...,h_L(\vx)$. Then we increment the counters at $\mA[i,h_i(\vx)]$ for all $i\in[1,...,L]$. So each array cell stores the number of data elements that have been hashed to the corresponding \textsc{LSH} bucket. 

The \textsc{KDE} of a query is roughly a measure of the number of nearby elements in the dataset. Hence, it can be estimated by averaging over hash values for all rows of \textsc{RACE}:
\[\hat{K}(\vq)=\frac{1}{L}\sum_{i=1}^L\mA[i,h_i(\vq)]\] 
For a query $\vq$, the \textsc{RACE} sketch computes the \textsc{KDE} using the median of means procedure rather than the average to bound the failure probability of the randomized query algorithm. The key result of \cite{luo2018arrays} is:
\begin{theorem}[\textsc{ACE} estimator]
    Given a dataset $\sD$ and an \textsc{LSH} family $\gH$ with finite range $[1, W]$, construct a hash function $h: x\rightarrow [1,W^p]$ by concatenating $p$ independent hashes from $\gH$. Suppose an \textsc{ACE} array $\mA$ is created by hashing each element of $\sD$ using $h(.)$. Then for any query $\vq$,
\[
\mathbb{E}[\mA[h(\vq)]] = \sum_{\vx \in \sD} k^p(\vx, \vq)
\]
\end{theorem}

\textsc{ACE} is useful for \textsf{KDE} because it is an unbiased estimator. Moreover, \cite{coleman2019sub} have shown a tight bound on the variance.
\begin{theorem}\label{ace_var}
    Given a query $\vq$, the variance of \textsc{ACE} estimator $\mA[h(\vq)]$ is bounded as:
    \[var(\mA[h(\vq)])\leq \left( \sum_{\vx\in\sD}k^{p/2}(\vx,\vq)\right)^2\]
\end{theorem}
This implies that we can estimate \textsf{KDE} using repeated \textsc{ACE} or \textsc{RACE} with very low relative error given a sufficient number of repetitions $p$.

\subsection{Exponential Histogram} 

We want to solve the \textsc{Basic Counting} problem for data streams with respect to the last $N$ elements seen so far. To be specific, we want to count the number of 1's in the last $N$ elements given a stream of data elements containing 0 or 1. \cite{datar2002maintaining} proposed an algorithm for this problem, which provides a $(1+\epsilon)$ estimate of the actual value at every instant. They use an \textsc{EH} to maintain the count of active 1's in that they are present within the last N elements. Every bucket in the histogram maintains the timestamp of the most recent 1, and the number of 1's is called the bucket size. The buckets are indexed as $1,2,...$ in decreasing order of their arrival times \ie the most recent bucket is indexed 1. Let the size of the $i^{th}$ bucket be denoted as $C_i$. When the timestamp of a bucket expires, we delete that bucket. The estimate for the number of 1's at any instant is given by subtracting half of the size of the last bucket from the total size of the existing buckets, 
\ie $(TOTAL - LAST)/2$. To guarantee counts with relative error of at most $\epsilon$, the following invariants are maintained by the algorithm (define $k$ as $\lceil 1/\epsilon\rceil$) :
\begin{enumerate}
    \item \textbf{Invariant 1:} Bucket sizes $c_1, c_2, \dots, c_m$ are such that $\forall j\leq m$ we have $\frac{c_j}{2(1 + \Sigma_{i=1}^{j-1} c_i)} \leq \frac{1}{k}$.
    \item \textbf{Invariant 2:} Bucket sizes are non-decreasing, i.e. $c_1 \leq c_2 \leq c_3 \leq \cdots \leq c_m$. Further, they are constrained to only powers of 2. Finally, for every bucket other than the last bucket, there are at least $\frac{k}{2}$ and at most $\frac{k}{2} +1$ buckets of that size.
\end{enumerate}
It follows from invariant 2 that to cover all active 1's, we need no more than $n\leq (\frac{k}{2}+1) \cdot (\log (2\frac{N}{k}+1)+1)$ buckets. The bucket size takes at most $\log N$ values, which can be stored using $\log \log N$ bits, and the timestamp requires $\log N$ bits. So the memory requirement for an EH is $\gO(\frac{1}{\epsilon}\log ^2N)$. By maintaining a counter each for $TOTAL$ and $LAST$, the query time becomes $\gO(1)$.

\section{Streaming (c,r)-Approximate Near Neighbor} \label{sec:ANN}

We generalize the classical Motwani--Indyk scheme \cite{indyk1998approximate} to the streaming setting by formulating the $(c,r)$-\textsf{ANN} problem under the assumption that the number of points inside any ball follows a Poisson distribution with an appropriate mean parameter. Within this framework, we prove that it suffices to retain only a sublinear sample of the data stream, specifically $\mathcal{O}(n^{1-\eta})$ points obtained through uniform random sampling.

\begin{figure}[H]
    \centering
    \includegraphics[height=0.15\paperheight,keepaspectratio]{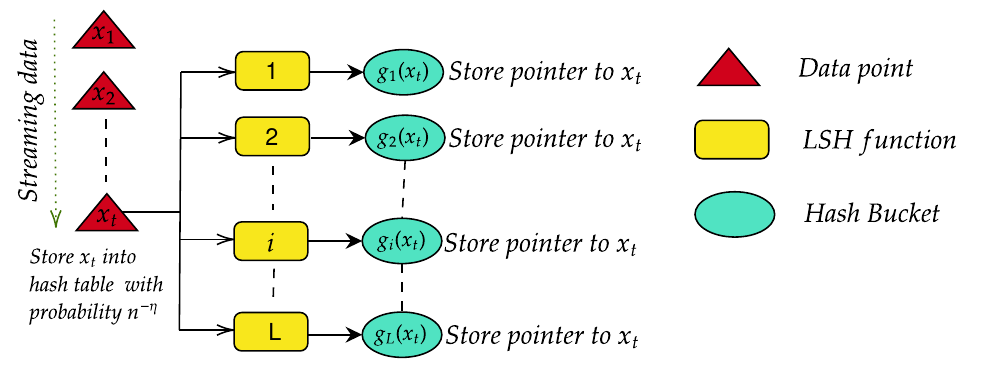}
    \caption{Insert mechanism of Algorithm~\ref{NN_algo} where, with probability $n^{-\eta}$, we store point $\vx_t$ into the hash table using $L$ independent hash functions $g_i(\vx)$.}
    \label{ann:insert}
\end{figure}

\begin{figure}[H]
\centering
  \centering
  \includegraphics[height=0.15\paperheight,keepaspectratio]{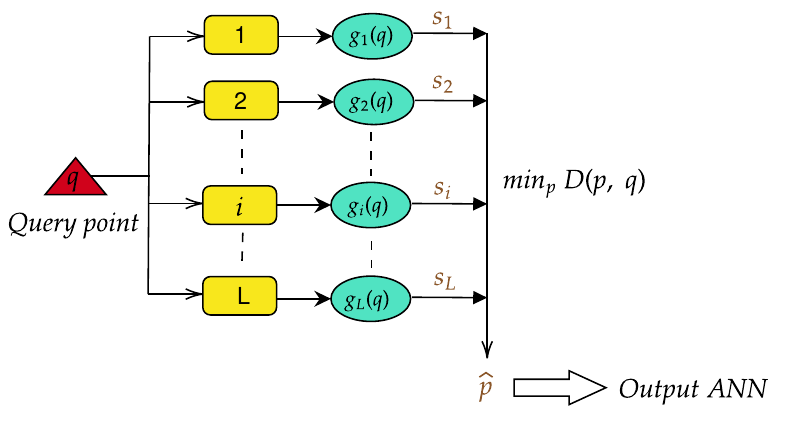}
\caption{Query mechanism of Algorithm~\ref{NN_algo} which retrieves the closest point to the query vector $\vq$ from a \emph{candidate set} $\sS$ constructed as the union of all hash bucket collision sets,\ie $\sS = \bigcup_{i=1}^{L} s_i$, where each $s_i$ denotes the set of points that collide with $\vq$ in the $i^{\text{th}}$ hash table.}
\label{ann:query}
\end{figure}

\subsection{The Algorithm}
Let $N$ be an upper bound on the size of the data stream $\sD$. We initialize a family of hash functions $\mathcal{H}$ with parameters $k$ and $L$, chosen as functions of $N$ and $\epsilon$. Below, we describe a scheme that inserts points from the stream into our data structure and employs the \textit{Query Processing} routine to solve the $(c,r)$-\textsf{ANN} problem for any query point $\vq$ (see Algorithm~\ref{NN_algo}). The insert and query mechanisms for our scheme are illustrated in Figures \ref{ann:insert} and \ref{ann:query}, respectively.

\begin{algorithm}[ht]
\DontPrintSemicolon
\SetAlgoNlRelativeSize{-1}
\caption{\textsf{S-ANN}: \textsf{ANN} sketch construction and querying for streaming.}\label{NN_algo}
\KwIn{Data stream $\mathcal{D}$, query point $\mathbf{q}$, \textsc{LSH} family $\mathcal{H}$, parameters $k$ \& $L$, sampling parameter $\eta$}
\KwOut{Approximate nearest neighbor $p^*$ or \texttt{NULL}}

Initialize $L$ independent hash functions $\{g_1, \dots, g_L\}$\;
\ForEach{$p \in \mathcal{D}$}{
  Decide whether to drop or store $p$ \tcp*{Uniform sample to store $O(n^{1-\eta})$ points}
  \For{$j = 1$ \KwTo $L$}{
    Insert $p$ into bucket $g_j(p)$\;
  }
}
\BlankLine
\textit{Query processing:}\;
Initialize candidate list $\mathcal{C} \gets \emptyset$\;
\For{$j = 1$ \KwTo $L$}{
  Retrieve all points from bucket $g_j(q)$ and add to $\mathcal{C}$\;
  \If{$|\mathcal{C}| \ge 3L$}{
    \textbf{break}\;
  }
}
Remove duplicates from $\mathcal{C}$\;
$p^* \gets \arg\min_{p_j \in \mathcal{C}} D(p_j, q)$\;
\If{$D(p^*, q) \le r_2$}{
  \Return $p^*$\;
}
\Return \texttt{NULL}\;
\end{algorithm}

\subsection{Correctness}
A query $\vq$ can succeed when certain events take place, which are mentioned here.
\begin{lemma} \label{eventlemma}
 Let $\sP$ be the set of points at the time when the query is executed. Define $B(\vp,r)$ as a ball of radius $r$ surrounding a point $\vp$. We define the following two events for any query point $\vq$:
    \begin{itemize}
        \item $\mathbf{E_1}:$ $\exists \vp' \in B(\vq, r)$ such that $g_j(\vp') = g_j(\vq)$ for some $j \in \{1, \dots, L\}$. 
        \item $\mathbf{E_2}:$ The number of points from $\sP - B(\vq, r_2)$ which hash to the same bucket as $\vq$ is less than $3L$ \ie
        \[
        \sum_{j=1}^{L} |(\sP - B(\vq, r_2)) \cap g_j^{-1}(g_j(\vq))| < 3L.
        \] 
    \end{itemize}
If $\mathbf{E_1}$ and $\mathbf{E_2}$ hold, then the query $\vq$ succeeds
\end{lemma}
\begin{proof}
    Denote the nearest neighbor of query $\vq$  as $\vp^*$. We have two cases: 
    
    \textbf{Case 1} $\exists \vp^* \in B(\vq, r)$: Since the algorithm stores $3L$ candidate points, $\mathbf{E_2}$ holding implies that the candidate set has a non-zero number of vacant spots. Now, if $\mathbf{E_1}$ also holds, we know that there exists $\vp' \in B(\vq, r)$ such that $g_j(\vp') = g_j(\vq)$ for some $j \in \{1, \dots, L\}$. So this $p'$ must get added to the candidate set. Since we ultimately return a point that is closest to the query point $\vq$, we guarantee that we return $\vp'$ or a closer point. But since $\vp' \in B(\vq,r)$, trivially we have that $\vp' \in B(\vq, cr)$, hence the query $\vq$ succeeds. 
    
    \textbf{Case 2} $\not\exists \vp^* \in B(\vq, r)$: In this case, the algorithm can return $null$ or any point in $\sP$, so the query $\vq$ succeeds trivially.
\end{proof}
It is to be noted that the aforementioned $p'$ does not need to be the nearest neighbor of the query point. It suffices to have a point present in a ball of radius $r$ surrounding the query after the random sampling. In the following lemma, we show how we can set the parameter $k$ to guarantee the success of $\mathbf{E_2}$ with high probability. 

\begin{lemma}\label{e2lemma}
    If we set $k= \lceil (1-\eta)\log_{1/p_2} n \rceil$, for a certain query $\vq$ event $\mathbf{E_2}$ succeeds with probability $P_2 \geq \frac{2}{3}$ if we store $\gO(n^{1-\eta})$ points from the stream independent of how the data is distributed.
\end{lemma}
\begin{proof}
    %  For any $\vp' \in \sP - B(\vq, cr)$, the probability of collision under a fixed $g_j$ is at most $p_2^k \leq p_2^{\log_{1/p_2}n}= \frac1n$. 
    % Since we store only $n^{1-\eta}$ points, the expected number of such collisions is at most $n^{-\eta}$ for a function $g_j$, and at most $L \cdot n^{-\eta}$ across all $L$ functions. 
    % Let $\mathbf{N}$  be a random variable denoting the number of collisions. Using Markov's Inequality, we can say that the probability that more than $3L$ such collisions occur is:
    % \begin{align*}
    %     P(\mathbf{N} \geq 3L) \leq \mathbb{E}[\mathbf{N}]/3L \leq \frac{L \cdot n^{-\eta}}{3L} = \frac{1}{3n^\eta}
    % \end{align*}
     
    %  Thus the success probability for event $\mathbf{E_2}$, $P_2 =1-P(\mathbf{N} \geq 3L)\geq 1 - \frac{1}{3n^\eta}$.
     For any $\vp' \in \sP - B(\vq, cr)$, the probability of collision under a fixed $g_j$ is at most $p_2^k \leq p_2^{(1-\eta)\log_{1/p_2}n}= n^{-(1-\eta)}$. 
    Since only $n^{1-\eta}$ points are stored, the expected number of such collisions is at most $n^{1-\eta} \cdot n^{-(1-\eta)} = 1$ for every $g_j$, and at most $L \cdot 1 =  L$ across all $L$ functions. 
    Let $\mathbf{N}$ denote the total number of collisions. Using Markov's Inequality, we can say that the probability that more than $3L$ such collisions occur is $P(\mathbf{N} \geq 3L) \leq \frac{\mathbb{E}[\mathbf{N}]}{3L} \leq \frac{L}{3L} = \frac{1}{3}$. Thus, $\mathbf{E_2}$ succeeds with probability $P_2 =1-P(\mathbf{N} \geq 3L)\geq \frac{2}{3}$.
\end{proof}

Now, we show that if we set $k$ as per Lemma \ref{e2lemma}, assuming that our data is obtained from a Poisson point process, we can guarantee the success of $\mathbf{E_1}$ with high probability for an appropriate choice of $L$.
\begin{lemma} \label{e1lemma}
    Assume that the points are distributed in such a manner that the number of points in every ball of radius $r$ is distributed as a Poisson random variable with mean $m$. Given that $k= \lceil (1-\eta)\log_{1/p_2} n \rceil$, if we set $L = n^{(1-\eta)\rho} / p_1$, for a certain query $\vq$ event $\mathbf{E_1}$ succeeds with probability $P_1 \ge (1 - e^{-mp})(1 - 1/e)$ on sampling $\gO(n^{1-\eta})$ points  from the stream.
\end{lemma}

\begin{proof}
    We require that $\exists \vp' \in B(\vq, r)$ such that $g_j(\vp') = g_j(\vq)$ for some $j \in \{1, \dots, L\}$.
First, consider the probability that there is at least one point retained in a ball of radius $r$ surrounding the query after the uniform sampling, \ie $\exists \vp' \in B(\vq, r)$. We know that the data follows a Poisson distribution, so if we say that the number of points in a ball of radius r surrounding a query is a Poisson random variable \textbf{K} with mean $m$: 
    \begin{align*}
        \mathbb{P}(\text{No points in the r-ball after uniform sampling}) &= \mathbb{E}[(1-p)^\textbf{K}] ,\textbf{K} \sim Poisson(m)\\ 
        &= \sum_0^{\infty} (1-p)^k \cdot e^{-m}\frac{m^k}{k!} \\ 
        &= e^{-m} \sum_0^{\infty} \frac{(m(1-p))^k}{k!} \\
        &= e^{-m} e^{m(1-p)} \\
        &= e^{-mp}
    \end{align*}
    where $p = n^{-\eta}$ is the probability that we keep the point in the data structure while uniformly sampling. This implies that the probability of having at least one point $(\vp')$ close to the query is $(1-e^ {-mp})$. 
    
    Now, given that there exists $\vp' \in B(\vq,r)$, we can lower bound the probability that $g_j(\vp') = g_j(\vq)$ for some $j \in \{1, \dots, L\}$ as follows:
    \begin{align}
        \label{eqn_local}
         P(g_j(\vp')=g_j(\vq))\geq p_1^k \geq p_1^{(1-\eta)\log_{1/p_2} n + 1} = p_1 n^{-(1-\eta)\log_{1/p_2} (1/p_1)} = p_1 n^{-(1-\eta)\rho}
    \end{align}
    
    Combining these two statements, the probability of success of event $\mathbf{E_1}$ is lower bounded as:
    
    \begin{equation*}
        \begin{split}
            P(\mathbf{E_1})&=P(\exists\vp'\in B(\vq,r)\text{ such that }g_j(\vp') = g_j(\vq) \text{ for some j }\in \{1, \dots, L\})\\
            &=P(\exists\vp'\in B(\vq,r)\wedge g_j(\vp') = g_j(\vq) \text{ for some j }\in \{1, \dots, L\})\\
            &=P(\exists\vp'\in B(\vq,r))\cdot P(g_j(\vp') = g_j(\vq) \text{ for some j }\in \{1, \dots, L\})\\
            &=(1 - e^{-mp})\cdot P\left(\bigcup_{j=1}^Lg_j(\vp') = g_j(\vq)\right)\\
            &=(1 - e^{-mp})\cdot \left(1-P\left(\bigcap_{j=1}^Lg_j(\vp') \neq g_j(\vq)\right)\right)\text{ (using \textit{De Morgan's Law})} \\
            &=(1 - e^{-mp})\cdot \left(1-\prod_{j=1}^LP(g_j(\vp') \neq g_j(\vq))\right)\\
            &=(1 - e^{-mp})\cdot \left(1-\prod_{j=1}^L\left(1-P(g_j(\vp')= g_j(\vq))\right)\right)\\
            &\geq (1 - e^{-mp})(1 - (1 - p_1 n^{-(1-\eta)\rho})^L)\text{ (using equation \ref{eqn_local})}
        \end{split}
    \end{equation*}
   By setting $L = n^{(1-\eta)\rho}/ p_1$,
   \begin{equation*}
   \begin{split}
       P(\mathbf{E_1}) = P_1 &\geq (1 - e^{-mp})(1 - (1 - p_1 n^{-(1-\eta)\rho})^{n^{(1-\eta)\rho}/ p_1})\\
       & \geq  (1 - e^{-mp})\{1 - (e^{-p_1n^{-(1-\eta)\rho}})^{n^{(1-\eta)\rho}/p_1}\}\\
       &=(1 - e^{-mp})(1 - 1/e)
   \end{split}
   \end{equation*}
     In the second inequality, we have used $1-x\leq e^{-x}$.
\end{proof}

\noindent We can now use the above Lemmas (\ref{eventlemma},~\ref{e2lemma},~\ref{e1lemma}) to prove the following theorem.

\begin{theorem} \label{sann}
Let $(\gX, \gD)$ be a metric space, and suppose that there exists an $(r, cr, p_{1}, p_{2})$-sensitive family $\mathcal{H}$, with $p_{1}, p_{2} \in (0,1)$, and define $\rho=\frac{\log(\frac{1}{p_1})}{\log(\frac{1}{p_2})}$. We further assume that the number of points contained in any ball of radius $r$ can be modeled as a Poisson random variable with mean $m$, where $m \geq Cn^ {\eta}$ for some constant $C > 0$. Then, for a point set $\sP \subseteq \gX$ comprising at most $n$ points, there exists a data structure for streaming $(c,r)$-nearest neighbor search with the following guarantees:
\begin{itemize}
    \item The data structure stores only $\gO(n^{1-\eta})$ points from the stream.  
    \item Each query requires at most $\gO\!\left({n^{\rho(1-\eta)}}/{p_{1}}\right)$
    distance computations and $\gO\!\left((1-\eta)\tfrac{n^{\rho(1-\eta)}}{p_{1}} \cdot \log_{1/p_{2}} n\right)$
    evaluations of hash functions from $\mathcal{H}$. The same bounds hold for updates.  
    \item The data structure uses at most $\gO(n^{(1-\eta)(1+\rho)} / p_1)$ words of space, in addition to the space required to store $\sP$. This is sub-linear in $n$ if $\eta > \frac{\rho}{1+\rho}$.
\end{itemize}
 The probability of failure is at most $\frac{1}{3} \;+\; \frac{e^{mp} + e - 1}{e^{mp+1}}$, which is less than 1 for an appropriate choice of \(m\).
\end{theorem}
\begin{proof}
  Assume that there exists $\vp^* \in B(\vq, r)$ (otherwise, there is nothing to prove). From Lemma \ref{eventlemma}, we can say that for a query to succeed, events $\mathbf{E_1}$ and $\mathbf{E_2}$ must occur with constant probability. 

From Lemma \ref{e2lemma} and Lemma \ref{e1lemma}, we can infer that for appropriate choice of $k$ and $L$, the query fails with probability at most $\frac{1}{3} + \frac{e^{mp} + e - 1}{e^{mp+1}}$. This proves the failure probability of the theorem. 
    
\end{proof}

\subsection{Extension to Batch Queries}
The result of Theorem~\ref{sann} extends naturally to the batch streaming setting, where a query 
consists of a set of points $\sQ = \{\vq_i\}_{i=1}^B$. Each batch can be viewed 
as $B$ independent queries, and the guarantees of Theorem~\ref{sann} apply to 
each of them. Moreover, the structure admits straightforward parallelization, making 
batch queries especially efficient in practice. 
\begin{corollary}
    The Streaming \textsf{ANN} data structure extends to the batch streaming setting, where a query 
consists of a set $\sQ = \{\vq_i\}_{i=1}^B$. In this case:
\begin{itemize}
    \item The data structure stores only $\gO(n^{1-\eta})$ points from the stream. 
    \item Each batch requires at most 
    $\gO\!\left(B \cdot {n^{\rho(1-\eta)}}/{p_{1}}\right)$ 
    distance computations and 
    $\gO\!\left(B \cdot (1-\eta)\tfrac{n^{\rho(1-\eta)}}{p_{1}} \cdot \log_{1/p_{2}} n\right)$ 
    evaluations of hash functions from $\mathcal{H}$.  
    \item The data structure uses at most $\gO(n^{(1-\eta)(1+\rho)} / p_1)$ words of space, in addition to the space required to store $\sP$.
\end{itemize}
The probability of failure of the batch is at most $B(\frac{1}{3}+\frac{e^{mp} + e -1}{e^{mp+1}})$, with each independent query failing with probability at most $\frac{1}{3}+\frac{e^{mp} + e -1}{e^{mp+1}}$.
\end{corollary}

\subsection{\textsf{ANN} in the Turnstile Model} In the turnstile setting, arbitrary deletions can break \textsf{ANN} guarantees if the nearest neighbor within a query ball is removed. To mitigate this, we assume an adversary can delete at most $d$ points from any ball of radius $r$. Under this restriction, the earlier sublinear-sample guarantees hold by bounding the probability that a ball contains at least $d+1$ points. We prove that, under natural assumptions on point distributions in a stream and limiting deletions per region, we can maintain a sublinear-sized data structure that supports efficient approximate nearest neighbor queries with low failure probability, handling both insertions and deletions. %(see Theorem~\ref{thm:ANN:turnstile}).

\begin{lemma}\label{poissontailbound} (Tail Bound for a Poisson random variable)
Let $\textbf{S} \sim \mathrm{Poisson}(\lambda)$ and $d \leq \lambda$. Then
$P(\textbf{S} \leq d) \leq e^{d - \lambda + d \ln\!\frac{\lambda}{d}}$.
\end{lemma}
\begin{proof}
For any $t \geq 0$, we can say $P(\textbf{S} \leq d) = P\!\big(e^{-t\textbf{S}} \geq e^{-td}\big)
\leq e^{td}\,\mathbb{E}\!\left[e^{-t\textbf{S}}\right]$ using Markov's inequality. We also know that the MGF of a Poisson random variable is $
\mathbb{E}\!\left[e^{-tS}\right]=e^{\lambda(e^{-t}-1)}$. Thus, putting it all together, we can say that,
\[
P(\textbf{S} \leq d) \leq e^{\lambda(e^{-t}-1)+td}
\]

Define $\varphi(t) := \lambda(e^{-t}-1)+td$. Differentiating and setting $\varphi'(t)=0$ gives $e^{-t} = d/\lambda$. Since $d \leq \lambda$, we obtain the optimum as $t^\star = \ln\!\frac{\lambda}{d}$. Substituting $t^\star$ into $\varphi(t)$ yields
\[
\varphi(t^\star) = \lambda\!\left(\tfrac{d}{\lambda}-1\right) + d \ln\!\frac{\lambda}{d}
= d - \lambda + d \ln\!\frac{\lambda}{d}
\]
Hence, we obtain the result stated in the lemma
\[
P(\textbf{S} \leq d) \leq e^{d - \lambda + d \ln\!\frac{\lambda}{d}}
\]
\end{proof}
\begin{lemma}\label{poissonthinning} (Poisson Thinning)
Let $\textbf{K}\sim\mathrm{Poisson}(m)$ be the number of points in a ball. Suppose each point is kept independently with probability $p\in[0,1]$ during random sampling. Let $\textbf{S}$ denote the number of points that remain after sampling. Then $\textbf{S} \sim \mathrm{Poisson}(mp)$.
\end{lemma}
\begin{proof}
    Conditioning on $\textbf{K}$, given $\textbf{K}=k$, we know that the number of retained points $\textbf{S}$ follows a $\mathrm{Binomial}(k,p)$ distribution:
\[
P(\textbf{S}=s\mid \textbf{K}=k)=\binom{k}{s} p^s(1-p)^{\,k-s}
\]
Thus, we can obtain the probability density function of $\textbf{S}$ as 
\begin{align*}
P(\textbf{S}=s) &=\sum_{k=s}^{\infty}P(\textbf{K}=k)P(\textbf{S}=s\mid \textbf{K}=k)\\
&=\sum_{k=s}^{\infty} e^{-m}\frac{m^k}{k!}\binom{k}{s} p^s(1-p)^{\,k-s} \\
&= \sum_{k=s}^{\infty}e^{-m}\frac{(mp)^s}{s!}\cdot\frac{\big(m(1-p)\big)^{\,k-s}}{(k-s)!}
\end{align*}

Summing over $k\ge s$ gives a Poisson tail series that sums to an exponential:
\[
P(\textbf{S}=s)
= e^{-m}\frac{(mp)^s}{s!}\sum_{t=0}^{\infty}\frac{\big(m(1-p)\big)^t}{t!}
= e^{-m}\frac{(mp)^s}{s!}\,e^{\,m(1-p)}
= e^{-mp}\frac{(mp)^s}{s!},
\]
which is the probability mass function for  $\mathrm{Poisson}(mp)$ . Therefore $\textbf{S}\sim\mathrm{Poisson}(mp)$.
\end{proof}

We use these results to prove the theorem for \textsf{ANN} under the turnstile model.

\begin{theorem}\label{thm:ANN:turnstile}
Let $(\gX, \gD)$ be a metric space, and suppose there exists an $(r, cr, p_{1}, p_{2})$-sensitive family $\mathcal{H}$, with $p_{1}, p_{2} \in (0,1)$, and define $\rho=\frac{\log(\frac{1}{p_1})}{\log(\frac{1}{p_2})}$. We further assume that the number of points contained in any ball of radius $r$ can be modeled as a Poisson random variable with mean $m$, where $m \geq Cn^ {\eta}$ for some constant $C > 0$. Assume that an adversary may delete up to $d$ points from any $r$-ball (strict turnstile) such that $d \leq mp$. Then, for a point set $\sP \subseteq \gX$ comprising at most $n$ points, there exists a data structure for turnstile streaming $(c,r)$-nearest neighbor search with the following guarantees:
\begin{itemize}
    \item The data structure stores only $\gO(n^{1-\eta})$ points from the stream.  
    \item Each query requires at most $\gO\!\left({n^{\rho(1-\eta)}}/{p_{1}}\right)$
    distance computations and $\gO\!\left((1-\eta)\tfrac{n^{\rho(1-\eta)}}{p_{1}} \cdot \log_{1/p_{2}} n\right)$
    evaluations of hash functions from $\mathcal{H}$. The same bounds hold for updates.
    \item The data structure supports arbitrary deletion of points as per the strict turnstile model (up to $d$ points from each r-ball) 
    \item The data structure uses at most 
    $\gO(n^{(1-\eta)(1+\rho)} / p_1)$ words of space, in addition to the space required to store $\sP$.
\end{itemize}

 The failure probability is at most $\frac{1}{3} + \frac{1}{e} + e^{d - mp + d\ln{\frac{mp}{d}}}(1 - \frac{1}{e})$, which is less than 1 for an appropriate choice of $C$ and $d$.
\end{theorem}
\begin{proof}
    The proof does not vary too much from that of the vanilla streaming case. For correctness, we still require Lemma \ref{eventlemma} to hold. It is easy to see that $\mathbf{E_2}$ as defined in Lemma \ref{eventlemma}, holds trivially on deletion of points under the turnstile model, because the probability of hashing far-away points strictly decreases on deleting points from the data structure.
    
    We need to show that $\mathbf{E_1}$ still holds with sufficiently high probability. We follow a similar approach to Lemma \ref{e1lemma} to show that after deletion of up to $k$ points, $\exists \vp' \in B(\vq, r)$ such that $g_j(\vp') = g_j(\vq)$ for some $j \in \{1, \cdots,L\}$. 
     We know that the original data follow a Poisson distribution, so if we say that the number of points in a ball of radius r surrounding a query is a Poisson random variable \textbf{K} with mean $m$, we can use Lemma \ref{poissonthinning} to say that the number of retained points follows a Poisson distribution with mean $mp$, where $p = n^{-\eta}$ is the probability that every point is retained independently. Denote this distribution by $\mathbf{S}$.
    \begin{align*}
        P(\text{At most $d$ points lie in an r-ball surrounding the query point})
        &= P(\mathbf{S} \leq d) \\
        &\leq e^{d - mp + d\ln{\frac{mp}{d}}}
    \end{align*}

This implies that the probability of having at least $d+1$ points close to the query is $(1-e^ {d - mp + d\ln{\frac{mp}{d}}})$. 

Now, given that there exists $\vp' \in B(\vq,r)$ even on deleting $k$ points in the worst case, we can lower bound the probability that $g_j(\vp') = g_j(\vq)$ for some $j \in \{1, \dots, L\}$ as follows:
    \begin{equation}
        \label{pcollision}
         P(g_j(\vp')=g_j(\vq))\geq p_1^k \geq p_1^{\log_{1/p_2} n + 1} = p_1 n^{-\log_{1/p_2} (1/p_1)} = p_1 n^{-\rho}
    \end{equation}
So, the worst case probability of success of event $\mathbf{E_1}$ is:
    \begin{equation*}
        \begin{split}
            P(\mathbf{E_1})&\geq (1 - e^{d-mp+d\ln{\frac{mp}{d}}})(1 - (1 - p_1 n^{-\rho})^L)\text{ (using equation \ref{pcollision})}
        \end{split}
    \end{equation*}
    Now, similar to Lemma \ref{e1lemma}, we can set $L = n^{(1-\eta)\rho}/ p_1$ to obtain $P_1 = (1 - e^{d-mp+d\ln{\frac{mp}{d}}})(1 - \frac{1}{e})$. So now, using Lemma \ref{e2lemma} and the success probability derived above, we can say that for an appropriate choice of $d$ and $L$, the guarantees of our data structure hold with failure probability $\frac{1}{3} + \frac{1}{e} + e^{d - mp + d\ln{\frac{mp}{d}}}(1 - \frac{1}{e})$.
\end{proof}

\subsection{Lower bound}

We briefly discussed in Section~\ref{sec:cont_ann} that without any assumption on the input distribution it is not possible to maintain any sublinear size sketch to provide provable guarantees for \textsf{ANN}. We prove this formally below. 

We consider the $(r,cr)$ decision problem on $n$ data points: given a
query $\vq$, distinguish whether some point lies within distance $r$
or all points lie beyond $cr$, with one case promised. We assume that
each nonempty queried $r$-ball contains $\Theta(m)$ points.

\begin{theorem}[Space lower bound]\label{thm:lower_bound}
Any one-pass streaming algorithm for this problem with per-query
success probability at least $2/3$ requires $\Omega(n/m)$ bits of memory.
\end{theorem}

% \begin{theorem}[Lower Bound]\label{thm:lower_bound}
%     Consider the $(r,cr)$ decision problem: given a query $\vq$, decide whether some data point lies within distance $r$ of $\vq$, or whether all data points lie beyond $cr$ (one of the two is promised). Suppose that each populated $r$-ball contains $\Theta(m)$ points. Then any one-pass streaming algorithm that solves this problem with constant success probability per query requires $\Omega\left(\frac{n}{m}\right)$ bits of memory, irrespective of the information retained by the algorithm or its encoding.
% \end{theorem}
\begin{proof}
    Let $k=\lfloor n/m\rfloor$. We work in $\mathbb{R}^d$ with $d=\Theta(\log k)$, which admits $3k$ points at pairwise distance greater than $2cr$ (e.g.\ a suitably scaled lattice); fix such a set and use it for the locations below. Given an arbitrary bit string $x\in\{0,1\}^k$, for each $i\in[k]$ fix a query point $\vq_i$, a nearby location $a_i$ within distance $r$ of $\vq_i$, and a far location $b_i$. For each bit $x_i$, insert a cluster of $m$ points near $a_i$ if $x_i=1$, and near $b_i$ if $x_i=0$. The stream then contains at most $n$ points, and every populated $r$-ball contains $\Theta(m)$, satisfying the promise. By the $2cr$ separation, $\vq_i$ has a data point within distance $r$ if and only if $x_i=1$; when $x_i=0$, every data point lies at distance greater than $cr$ from $\vq_i$, since a point near $a_j$ for $j\neq i$ is at distance greater than $2cr-r>cr$ from $\vq_i$.

This construction yields a one-way communication protocol for $\mathrm{INDEX}$. Alice runs the streaming algorithm on the point set encoding $x$ and sends its memory state to Bob. Given an index $i$, Bob issues the single query $\vq_i$; a correct near/far answer reveals $x_i$. (The algorithm returns an actual stored point and can verify $D(\vp^\ast,\vq)\leq cr$, so it solves the decision problem; our sketch does exactly this.) Since Bob makes only one query, the protocol inherits the algorithm's constant success probability, and any randomized one-way protocol for $\mathrm{INDEX}$ on $k$ bits requires $\Omega(k)$ bits. Hence the streaming algorithm must use $\Omega(k)=\Omega(n/m)$ bits of memory. 
\end{proof}

\begin{remark}[Consequences of the space lower bound]
Theorem~\ref{thm:lower_bound} implies that sparse neighborhoods,
with $m=O(1)$, require $\Omega(n)$ bits of memory, ruling out
$o(n)$-bit one-pass algorithms in the worst case. More generally,
when $m=\Theta(n^\eta)$, the lower bound becomes
$\Omega(n^{1-\eta})$ bits. This bound applies to arbitrary
retention policies, including adaptive and data-dependent ones,
and does not rely on the Poisson assumption.

The dependence on $n$ coincides with that of our sample size,
although the lower bound is measured in bits rather than retained
points. Up to logarithmic factors and point-representation costs,
the gap between our space bound and this lower bound is the
$n^{\rho(1-\eta)}$ factor contributed by the LSH tables.
This overhead reflects the repeated-table construction
of~\cite{har2012approximate}. Reducing this gap while preserving
our query and update guarantees is a natural direction for
future work.
\end{remark}

%This lower bound gives two immediate regimes. When the neighborhoods are sparse, namely $m=O(1)$, the space requirement is $\Omega(n)$, ruling out any sublinear-space algorithm. More generally, when $m=\Theta(n^\eta)$, the lower bound becomes $\Omega(n^{1-\eta})$. This matches, up to constant factors, the number of sampled points retained by our algorithm. So the sampling rate is optimal, not merely a byproduct of the stochastic assumption: any algorithm, under any retention policy including adaptive ones, needs at least this much memory once occupancy is fixed at $m$. Total space is optimal up to the $n^{\rho(1-\eta)}$ LSH-table factor, which is already present in the structure of~\cite{har2012approximate} and open to close even there. 

\section{Sliding-Window Approximate Kernel Density Estimation (\textsf{A-KDE})}\label{sec:akde}
In \cite{coleman2020sub}, the authors propose \textsc{RACE}, an efficient sketching technique for kernel density estimation on high-dimensional streaming data. We have seen that we can get low relative errors using a larger number of repetitions in \textsc{RACE}. We propose Sliding window \textsf{AKDE}, in short \textsf{SW-AKDE}, which uses a modified  \textsc{RACE} structure to make it suitable for a sliding window model by using \textsc{EH} \cite{datar2002maintaining}. We give bounds on the number of repetitions \ie, the number of rows, to obtain a good estimate of the \textsf{KDE} with high probability.
\subsection{The Algorithm} In \textsc{RACE}, we increment $\mA[i,h_i(\vx)]$ for every new element $\vx$ coming from dataset $\sD$. In the sliding window model, we are interested in the last $N$ elements, assuming that we get an element every time step. Hence, we have to find the number of times a counter has been incremented in the last $N$ time steps. This problem is similar to the \textsc{Basic Counting} problem in \cite{datar2002maintaining}, where the incoming stream of data is 1 if the counter is incremented at a time instant, otherwise 0. We will store an \textsc{EH} for each cell of \textsc{RACE}. On querying the \textsc{EH}, we will get an estimate of the count in a particular cell. In the \textsc{RACE} sketch, the estimator is computed using the median of means procedure. For \textsf{SW-AKDE}, we will take the average of \textsc{ACE} estimates over L independent repetitions. The \textsf{SW-AKDE} algorithm is given in \ref{KDE_algo}.

\begin{figure}[ht]
    \centering
    \includegraphics[height=0.15\paperheight,keepaspectratio]{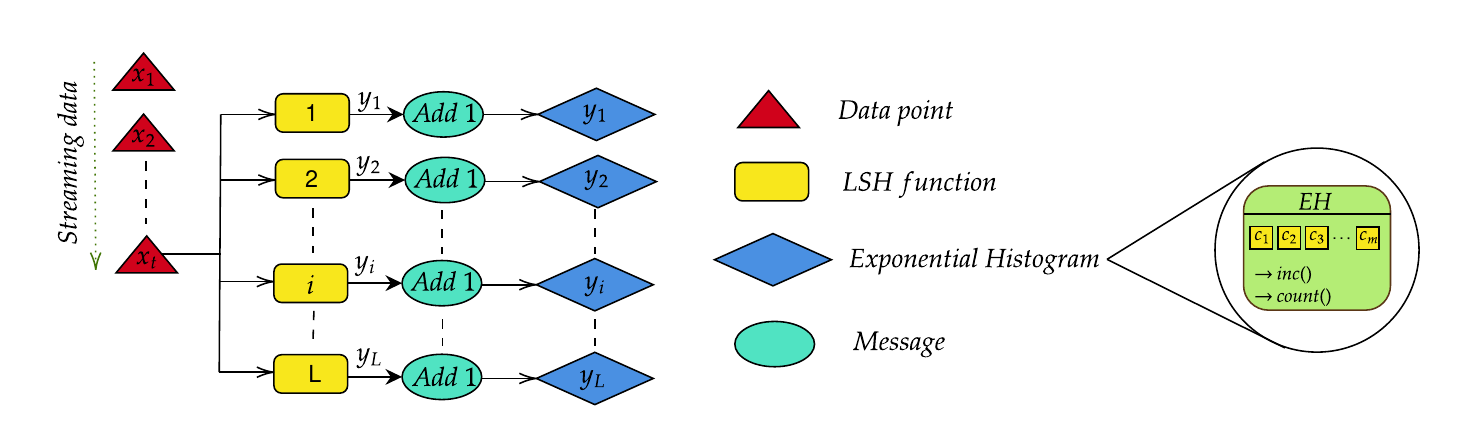}
    \caption{Update mechanism of Algorithm~\ref{KDE_algo}, where \(y_i\) is the output of the \textsc{LSH} function \(i\) on \(x_t\). For the \(i^{th}\) row, we add 1 to the \textsc{EH} at index \(y_i\)}
    \label{kde:sketch1}
\end{figure}

\begin{figure}[ht]
\centering
\includegraphics[height=0.15\paperheight,keepaspectratio]{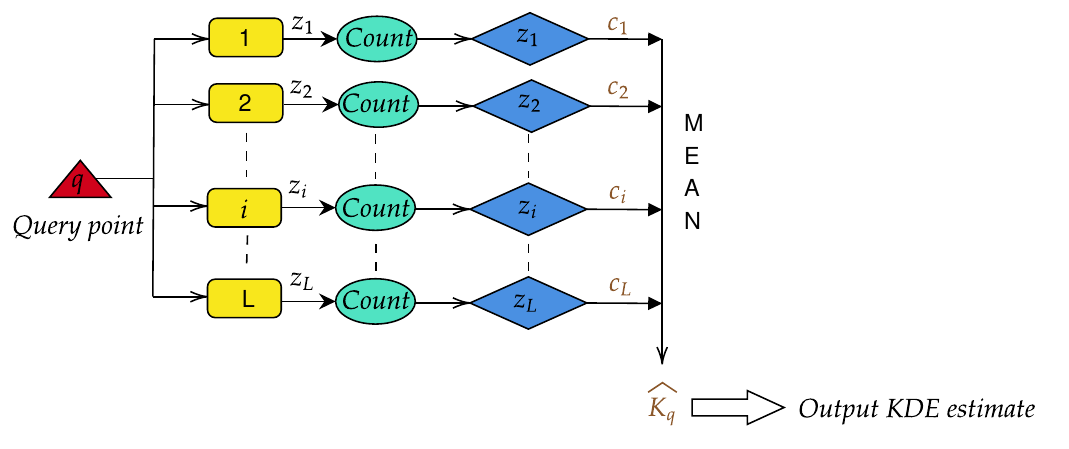}
\caption{Query mechanism of Algorithm~\ref{KDE_algo} where \(z_i\) is the output of the \textsc{LSH} function \(i\) on \(q\). For the \(i^{th}\) row, \(c_i\)  is the count 
estimate of the \textsc{EH} at index \(z_i\). }
\label{kde:sketch2}
\end{figure}

\begin{algorithm}[h]
\DontPrintSemicolon
\SetAlgoNlRelativeSize{-1}
\caption{\textsf{SW-AKDE}: \textsf{AKDE} sketch construction and querying for sliding window}
\label{KDE_algo}

\textit{Preprocessing:}\;
\KwIn{Data set $\mathcal{D}$, \textsc{LSH} family $\mathcal{H}$ of range $W$, parameters $p,\, k,\, L$ and $N$}
\KwOut{RACE sketch $A$ of size $\mathbb{Z}^{L\times W^p}$}

Initialize $L$ independent hash functions $\{h_1, h_2, \dots, h_L\}$, where each $h_i$ is constructed by concatenating $p$ independent hashes from $\mathcal{H}$\;
$A \gets \textsc{empty}$ \tcp*[f]{\textsc{RACE} structure}\;
$t \gets 0$ \tcp*[f]{timestamp initialized to 0}\;

\ForEach{$p \in \mathcal{D}$}{
  \For{$i = 1$ \KwTo $L$}{
    $j \gets h_i(p)$\;
    \eIf{$A[i,j]$ is empty}{
      Create an Exponential Histogram with parameters \(k,N\) at $A[i,j]$ with timestamp $t$\;
    }{
      Add a 1 to the Exponential Histogram at $A[i,j]$ with timestamp $t$\;
    }
  }
  $t \gets t + 1$\;
}
\BlankLine
\textit{Query processing:}\;
\KwIn{RACE sketch $A$, query $\mathbf{q}$, hash functions $\{h_1, h_2, \dots, h_L\}$ initialized in preprocessing}
\KwOut{Approximate \textsf{KDE} estimate $y$}

$y \gets 0$ \tcp*[f]{initialize the \textsf{KDE} estimate to 0}\;
\For{$i = 1$ \KwTo $L$}{
  \If{$A[i, h_i(\mathbf{q})]$ is not empty}{
    $c \gets$ estimate of count in the Exponential Histogram at $A[i, h_i(\mathbf{q})]$\;
    $y \gets y + c$\;
  }
}
$y \gets y / L$ \tcp*[f]{compute the approximate \textsf{KDE}}\;
\Return $y$\;
\end{algorithm}

% \vspace{-2mm}
\subsection{Correctness}
The algorithm is illustrated in Figs.~\ref{kde:sketch1} and~\ref{kde:sketch2}.
Consider the \textsc{ACE} Estimator. We know that if we have the actual counts (by using a counter instead of \textsc{EH}), then $X = \mA[h(\vq)]$ is an unbiased estimator for the Kernel Density with bounded variance(\ref{ace_var}). So, $\textbf{K}=\mathbb{E}[X]$ is the Kernel Density estimate. We will now demonstrate that, by using \textsc{EH}, the new estimator for a single \textsc{ACE} instance approximates $\textbf{K}$ with a certain level of accuracy.
\begin{lemma} \label{lem:4.1}
Let Y be the new estimator obtained from querying the \textsc{EH} at $\mA[h(\vq)]$. Then, $\mathbb{E}[Y]\leq (1+\epsilon')\textbf{K}$.\end{lemma}
\begin{proof}
    Suppose the relative error of the estimate from the \textsc{EH} algorithm is $\epsilon'$. So,
    \begin{equation}\label{eqnKD1}
    \begin{split}
        &|Y-X|\leq \epsilon'X \\
        &\text{ Taking expectation on both sides,} \\
        \implies&\mathbb{E}[|Y-X|]\leq \epsilon'\mathbb{E}[X]=\epsilon'\textbf{K} \text{  (since $\textbf{K}=\mathbb{E}[X]$)} \\
        \implies& |\mathbb{E}[Y]-\mathbb{E}[X]|\leq\mathbb{E}[|Y-X|]\leq \epsilon'\textbf{K}\\
        \implies& |\mathbb{E}[Y]-\textbf{K}|\leq \epsilon'\textbf{K} \\
        \implies& (1-\epsilon')\textbf{K}\leq \mathbb{E}[Y]\leq (1+\epsilon')\textbf{K}\\
    \end{split}
\end{equation}
%\vspace{-.2mm}
\end{proof}

We use $L$ independent \textsc{ACE} instances to estimate the kernel density. Hence, the \textsf{KDE} estimator in the current setting is,
\[\hat{Y} = \frac{1}{L}\sum_{i=1}^LY_i\quad \text{where $Y_i$ is the \textsc{EH} estimate for $\mA[h_i(\vq)]$} \] The expectation of $\hat{Y}$ is $\mathbb{E}Y_i$ and the variance of $\hat{Y}$ is $\frac{1}{L}Var(Y_i)$.
% \begin{align*}
%  \mathbb{E}\hat{Y} &=\mathbb{E}\left(\frac{1}{r}\sum_{i=1}^rY_i\right)=\mathbb{E}Y_i   \\
%  Var(\hat Y)&=Var\left(\frac{1}{r}\sum_{i=1}^rY_i\right)=\frac{1}{r^2}rVar(Y_i)=\frac{1}{r}Var(Y_i)
% \end{align*}
Now we will show the bounds of the estimator $\hat Y$ and derive the necessary bounds on $L$.
\begin{lemma}\label{yhatbound}
        $|\hat{Y} - \mathbb{E}\hat{Y}| < \epsilon ' \mathbb{E}[\hat{Y}]$ holds with probability $1-\delta$ if:
    \[ L \geq \frac{2\max\{X_i\}^2}{(1 - \epsilon')^2\textbf{K}^2}\log\left(\frac{2}{\delta}\right)\]
    where $\epsilon'$ is the relative error of \textsc{EH}, $L$ is the number of repetitions of \textsc{ACE} (or the number of rows in the \textsf{SW-AKDE} sketch).
\end{lemma}
\begin{proof}
    From \ref{eqnKD1}, $ |Y_i-X_i|\leq \epsilon'X_i\implies X_i(1 - \epsilon')\leq Y_i \leq X_i(1+\epsilon') \forall i\in\{1,2,\cdots,L\}$. Here, $X_i$ refers to the actual count in a cell in the last $N$ time steps, if we had used a counter instead of an \textsc{EH} for the $i^{th}$ hash function. 
    
    Note that, in order to estimate the lower bound for $L$, we need the maximum of $X_i$'s for $i\in[L]$. 
    Therefore, we can only arrive at the solution using an iterative technique, which can be made faster through usual binary trick. 
    The procedure is illustrated in algorithm ~\ref{r_algo}. This implies that the \textit{for} loop in the subroutine repeats $\log L$ times, where $L$ is the solution. Hence, we incur a logarithmic factor in the running time.
    
    It follows from \textit{Hoeffding's inequality}, where we define $\epsilon$ as $\epsilon'\mathbb{E}[\hat{Y}]$,
\begin{align*}
\mathbb{P}(|\hat{Y} - \mathbb{E}\hat{Y}| > \epsilon ' \mathbb{E}[\hat{Y}]) &\leq 2 \exp\left(-\frac{2L^2\epsilon^2}{\sum_{i=1}^L (b_i - a_i)^2}\right)  \\
    &\leq 2 \exp\left(-\frac{2\epsilon'^2\mathbb{E}[\hat{Y}]^2L^2}{\sum_{i=1}^L(2\epsilon'X_i)^2}\right) \text{ (using $\epsilon=\epsilon'\mathbb{E}[\hat{Y}]$)}\\
    &\leq 2 \exp\left(-\frac{2\epsilon'^2\mathbb{E}[\hat{Y}]^2L^2}{L(2\epsilon'\max\{X_i\})^2}\right) \\
    &\leq 2\exp\left(-\frac{L\mathbb{E}[Y_i]^2}{2\max\{X_i\}^2}\right) \text{ (using $\mathbb{E}[\hat{Y}]=\mathbb{E}[Y_i]$)}\\
\end{align*}
To bound this probability by $\delta$, we need : 
\begin{align*}
 &2\exp\left(-\frac{L\mathbb{E}[Y_i]^2}{2\max\{X_i\}^2}\right) \leq \delta \\
    &    \implies L \geq \frac{2\max\{X_i\}^2}{\mathbb{E}[Y_i]^2}\log\left(\frac{2}{\delta}\right) \\
    & \implies L \geq \frac{2\max\{X_i\}^2}{(1 -\epsilon')^2\textbf{K}^2}\log\left(\frac{2}{\delta}\right)\text{ (using lower bound of $\mathbb{E}[Y_i]$ from Lemma ~\ref{lem:4.1})}
\end{align*}
Thus $|\hat{Y} - \mathbb{E}\hat{Y}| \leq \epsilon ' \mathbb{E}[\hat{Y}]$ holds with probability $1-\delta$ if $L$ satisfies the aforesaid bound.
\end{proof}
Now we will show that $\hat{Y}$ gives a multiplicative approximation of the \textsf{KDE} with probability $1-\delta$.

\begin{lemma}\label{kdeestimate}
    The estimator $\hat Y$ gives a $(1+\epsilon)$ approximation of the  Kernel density with probability $1-\delta$.
\end{lemma}
\begin{proof}
    Let the \textsf{KDE} be given by \textbf{K}. The estimator from the \textsf{SW-AKDE} algorithm is $\hat Y$. Then,
   \begin{align*}
    |\hat{Y} - \textbf{K}| &\leq |\hat{Y}
 - \mathbb{E}\hat{Y}| + |\mathbb{E}\hat{Y} - \textbf{K}| \quad\text{ (using triangle inequality)}\\
 &\leq \epsilon'\mathbb{E}\hat{Y} + \epsilon'\textbf{K} \quad\text{ (using lemma \ref{yhatbound} and penultimate inequality of equation \ref{eqnKD1})}\\
 &\leq \epsilon'(1 + \epsilon')\textbf{K} + \epsilon'\textbf{K}\quad \text{ (using last inequality of equation \ref{eqnKD1})}\\
 &\leq (2\epsilon'+\epsilon'^2)\textbf{K} \\
 &= \epsilon\textbf{K} \quad\text{ (substituting $\epsilon=2\epsilon'+\epsilon'^2$)}\\
 \implies |\hat Y-\textbf{K}|&\leq \epsilon\textbf{K}
\end{align*}
Note that the bound from Lemma \ref{yhatbound} holds with probability $1-\delta$. Hence, this result holds with probability $1-\delta$.
\end{proof}
\begin{algorithm}[ht]
\DontPrintSemicolon
\SetAlgoNlRelativeSize{-1}
\caption{Find optimal $L$ for \textsf{AKDE} sketch}
\label{r_algo}
\KwIn{relative error of \textsc{EH} $\epsilon'$, $\delta$, Data set $\mathcal{D}$, \textsc{LSH} family $\mathcal{H}$ of range $W$, parameters $p$}
\KwOut{Optimal number of rows $\hat{L}$}

$L \gets l_0$ \tcp*[f]{Initialize $L$ to some value say 1}\;
\For{$i = 1,2,\cdots$}{
  $A \gets Preprocessing(\mathcal{D},\mathcal{H},p=1,L)$\tcp*[f]{call the preprocessing routine of algorithm~\ref{KDE_algo} by using counters instead of \textsc{EH} in the sketch update step}\;
  \eIf(\tcp*[f]{check stopping condition}){$L > \frac{2\max\{X_i\}^2}{(1 - \epsilon')^2\mathbf{K}^2}\log\left(\frac{2}{\delta}\right)$}{
    \textbf{break}\;
  }{
    $L\gets 2L$\tcp*[f]{Double the value of $L$}\;
  }
}
\KwRet{$L$}\;
\end{algorithm}
Let us compute the space requirement of the sketch proposed for our algorithm.
\begin{lemma}\label{race_space}
    The proposed \textsf{SW-AKDE} has space complexity $\mathcal{O}\left(LW^p \cdot \frac{1}{\sqrt{1+\epsilon} - 1} \log^2 N\right)$ where \(L\) is the number of rows, $W$ is the range of the hash function, \(p\) is the number of hash functions concatenated, $\epsilon$ is the relative error for \textsf{KDE}, $N$ is the window size.
\end{lemma}
\begin{proof}
     The number of cells in \textsf{SW-AKDE} sketch is $LW^p$. Each cell is represented by an \textsc{EH} of space complexity $\gO(\frac{1}{\epsilon'}\log ^2N)$ where $\epsilon'$ is the relative error of the \textsc{EH}. The relative error for \textsf{KDE} $\epsilon$ is related to $\epsilon'$ as (using lemma \ref{kdeestimate}):  
    \(\epsilon=2\epsilon'+\epsilon'^2\implies\epsilon'=\sqrt{1+\epsilon}-1\).\\
    Hence, the total space requirement for \textsf{SW-AKDE} is:
    \[LW^p\gO(\frac{1}{\epsilon'}\log ^2N)=LW^p\gO(\frac{1}{\sqrt{1+\epsilon}-1}\log ^2N)=
    \gO\left(LW^p\frac{1}{\sqrt{1+\epsilon}-1}\log ^2N\right)\]
\end{proof}

Using the above lemmas, we can state the main theorem as follows
\begin{theorem}\label{kde_result}
    Suppose we are given an \textsc{LSH} function with range $W$. Then the proposed \textsf{SW-AKDE} sketch with 
\(
 L = \mathcal{O} \left( \frac{2\max\{X_i\}^2}{(1 - \epsilon')^2\textbf{K}^2}\log\left(\frac{2}{\delta}\right) \right)
\) independent repetitions of the hash function provides a $1 + \epsilon$ multiplicative approximation to $\textbf{K}$ (which is the \textsf{KDE}) with probability $1 - \delta$, using space $\mathcal{O}\left(LW^p \cdot \frac{1}{\sqrt{1+\epsilon} - 1} \log^2 N\right)$ where N is the window size.
\end{theorem}

\subsection{Extension to batch queries} We can extend the result of Theorem ~\ref{kde_result} to the batch queries setting.
We define the dynamic streaming dataset where a batch of data points at a new timestamp $t$ is denoted by $\sX^{(t)}=\{\vx_i^{(t)} \in\mathbb{R}\}_{i=1}^{n_t}$. Let the batch size ($n_t$) be a constant, say $R$. For the sliding window setting, we will consider the last $N$ batches for the \textsf{KDE} estimation, instead of $N$ data points.\par
Our algorithm can be extended for this setting accordingly. We have to modify only the update step in the \textsc{EH}. The \textsc{EH} has to estimate the number of elements in the last $N$ batches which hash to $h_i(\vq)$. The maximum increment for an \textsc{EH} in a cell at a given time step is $R$. This will happen when all the elements in the current batch of size $R$ hash to the same \textsc{LSH} bucket. \cite{datar2002maintaining} show that the \textsc{EH} algorithm can be generalized for this problem using at most $(k/2+1)(\log(\frac{2NR}{k}+1)+1)$ buckets, where $k=\lceil\epsilon\rceil$. The memory requirement for each bucket is $\log N+\log(\log N+\log R)$ bits. Hence, we get the following corollary from Theorem \ref{kde_result}.
\begin{corollary}
      Suppose we are given an \textsc{LSH} function with range $W$. The data comes in batches of size $R$ at every time step. Then the proposed \textsf{SW-AKDE} sketch with 
\[
 L = \mathcal{O} \left( \frac{2\max\{X_i\}^2}{(1 - \epsilon')^2\textbf{K}^2}\log\left(\frac{2}{\delta}\right) \right)
\]

independent repetitions of the hash function provides a $1 + \epsilon$ multiplicative approximation to $\textbf{K}$ (which is the \textsf{KDE}) with probability $1 - \delta$, using space $\mathcal{O}\left(LW \cdot \frac{1}{\sqrt{1+\epsilon} - 1} \allowbreak(\log(2NR\sqrt{1+\epsilon})\allowbreak (\log N+\log(\log N+\log R))\right)$ where N is the window size.
\end{corollary}

\section{Experiments}\label{sec:exp}

We now present the empirical evaluation of our proposed sketching algorithms to validate their theoretical guarantees and to assess their practical performance across standard benchmark datasets\footnote{Our codes are available in \href{https://github.com/VedTheKnight/Streaming-ANN-and-Sliding-Window-KDE.git}{https://github.com/VedTheKnight/Streaming-ANN-and-Sliding-Window-KDE.git}}.

\subsection{Experiments for \textsf{ANN}} We evaluate the efficacy of our streaming \textsf{ANN} approach on standard benchmarks, focusing on the trade-offs between sampling aggressiveness (parameter $\eta$) and the approximate recall/accuracy of $(c,r)$-\textsf{ANN} queries. We also investigate the interplay between $\epsilon$ and $\eta$, demonstrating that for sufficiently large $\epsilon$, sub-linear sketch sizes are attainable with $\eta > \rho$. Finally, we assess the query time performance of our scheme in comparison to the baseline to understand its computational efficiency.

\emph{Datasets.} Experiments were conducted on two standard \textsf{ANN} benchmarks \cite{aumuller2020ann}: \texttt{sift1m} \cite{jegou2011product} (1M vectors, 128-dimensions) and \texttt{fashion-mnist} \cite{xiao2017fashion} (60,000 images, 784-dimensions), and a synthetic dataset\footnote{A synthetic dataset generated by sampling points, uniformly randomly from a Poisson point process (PPP).}: \texttt{syn-32} (100,000 points, 32-dimensions).

\emph{Implementation.} All data structures were implemented in Python, assuming \texttt{float32} vectors. Compression is measured relative to $N \times d \times 4/1024^2$ MB. No additional memory optimizations were applied. We used the $p$-stable scheme described in~\cite{datar2004locality} for hashing.

\emph{Baseline.} We compare against Johnson-Lindenstrauss (JL) projection \cite{johnson1984extensions}, the only known strict one-pass solution for $(c,r)$-\textsf{ANN}.

\emph{Experimental Setup.} We performed five experiments to verify the effectiveness of our scheme: 
\begin{enumerate}
\item \textbf{Clustering in popular benchmarks:} In high dimension, fixed-radius neighborhoods can be sparse unless the data has sufficient local structure. To examine whether this issue arises on standard \textsf{ANN} benchmarks, we measured the number of data points lying within distance $r$ of a typical query (see Table~\ref{tab:occupancy}). We chose $r$ conservatively, requiring that at most $5\%$ of the dataset lie within the acceptance radius $cr$; otherwise, a large fraction of the dataset would constitute valid answers, making the measurement meaningless. 

    \item \textbf{Comparison with RACE-CMS:}  We implemented the RACE-CMS sketch
(Algorithms~1--2 in~\cite{coleman2019sub}) and evaluated it against our scheme on SIFT
($n = 2\cdot10^{5}$--$10^{6}$) and Fashion-MNIST ($n = 5.5\cdot10^{4}$), on a
shared insert/query split with exact ground truth, sweeping both methods over
their own hyper-parameters and reporting each one's Pareto frontier.
    \item \textbf{Comparison with JL:} We compared our method and the JL baseline by sweeping over $\epsilon=0.5$ to $1$ and adjusting compression rates via $k$ (JL) and $\eta$ (ours). Each run stored 50,000 points and issued 5,000 queries with $r=0.5$. Metrics included approximate recall@$50$, $(c,r)$-\textsf{ANN} accuracy, and memory usage.
    \item \textbf{Sketch Size Scaling:} Using the \texttt{sift1m} dataset, we fixed $\epsilon=0.5$ and varied $\eta$ ($0.2$ to $0.8$) and dataset size $N$ ($1{,}000$ to $160{,}000$), measuring sketch size.
    \item \textbf{Query Time Analysis:} We evaluated the query-time performance of both \textsf{JL} and our \textsf{S-ANN} scheme by fixing $\epsilon=0.5$ and storing $10{,}000$ points with $100$ query points. For \textsf{JL}, we varied the projection dimension $k$, while for \textsf{S-ANN}, we varied the sampling parameter $\eta$.
    For each configuration, we measured the average recall and query throughput (queries per second). 

\end{enumerate}
\begin{table}[h]
\centering
\caption{Measured $r$-neighbourhood occupancy.}
\label{tab:occupancy}
\begin{tabular}{lrrrr}
\toprule
dataset & vector length & effective dim. & points within $r$ & largest usable $\eta$ \\
\midrule
fashion-mnist    & 784 & 15.4 & 177 & 0.47 \\
mnist            & 784 & 13.9 & 229 & 0.50 \\
sift1m  & 128 & 22.7 &  805 & 0.48 \\ 
uniform random   & 128 & 63.9 & \phantom{00}0 & --- \\
\bottomrule
\end{tabular}
\end{table}

\begin{table}[h]
\centering
\small
\caption{Query throughput (queries/second, single-query, identical workload).}
\label{tab:qps}
\begin{tabular}{llccc}
\toprule
Dataset & $n$ & RACE-CMS & Ours & speed-up \\
\midrule
Fashion-MNIST & $5.5\cdot10^{4}$ & 413 & 9{,}300 & $23\times$ \\
SIFT & $2\cdot10^{5}$ & 121 & 10{,}600 & $87\times$ \\
SIFT & $5\cdot10^{5}$ & 47 & 10{,}400 & $222\times$ \\
SIFT & $10^{6}$ & 24 & 9{,}500 & $398\times$ \\
\bottomrule
\end{tabular}
\end{table}
\paragraph{Discussion.} 

We demonstrate through the first experiment that the data in these popular benchmarks is clustered in a manner which can be modeled conveniently by the Poisson occupancy assumption. From the measurements in Table~\ref{tab:occupancy}, we observe that the premise that a fixed-radius ball in high dimensions is essentially empty is usually not observed on standard benchmarks: at the above radius, a typical query has a large number of points within $r$ across datasets. The exponential growth is governed by the number of directions the data actually varies in (around 15-20 here) rather than by the length of the vectors, which is 128–784. Image embeddings concentrate near much lower-dimensional structures, which is why neighborhoods stay populated. At this scale, this comfortably enables hyperparameter settings that attain sublinear space bounds. Additionally, The condition gets easier with scale. $r$ is fixed in advance while points keep arriving, so occupancy grows in proportion to $n$, and the same measurements give a usable $\eta$ of about 0.6 at a million points and 0.7 at $10^8$. The assumption is therefore tightest at benchmark scale and comfortable at the stream sizes that actually motivate the problem, consistent with sublinear sketches attaining good accuracy from $\eta\ge0.5$ (see Fig.~\ref{fig:memory-n}). Uniformly random 128-dimensional points have nothing within $r$, this is a sparse regime. From the lower bound (\ref{thm:lower_bound}), it is a regime where no one-pass algorithm succeeds, ours included. The occupancy condition is therefore not something we impose for convenience, it is the line between instances that are solvable in sublinear space and instances that provably are not.

Prior to our work, 
RACE-CMS~\cite{coleman2019sub} is the only known algorithm for streaming ANN. We show through these experiments that our method operates in a different regime - it recovers the full similarity vector $s(q)$ by
compressed sensing, which requires $s(q)$ to be sparse. On the dense benchmarks
we consider it separates from our sketch on both axes.
As shown in Fig.~\ref{fig:memory_recall_cms}, RACE-CMS achieves only marginal improvements in recall across the full range of sketch sizes. In contrast, \textsf{S-ANN} consistently exceeds this apparent performance ceiling for every tested value of $\epsilon$ and $n$, without requiring additional memory. It attains near-perfect approximate recall while retaining a compressed representation of the data. To verify that the RACE-CMS plateau is not caused by insufficient hyperparameter tuning, we additionally evaluate its recall over sufficiently broad ranges of the LSH power $K$ and kernel bandwidth values.
The difference is also structural at query time, and
Table~\ref{tab:qps} shows it widening with $n$. RACE-CMS stores no vectors, so
every query must score all $n$ identities,
$\hat{s}[i] = \min_j \widehat{\mathrm{CMS}}[j, h_j(i)]$, costing
$O(n \cdot \mathrm{rows})$ regardless of sketch size. Its throughput therefore falls by $17\times$ over the range we test,
while ours is nearly flat in $n$, because a query probes a bounded number of buckets
rather than the whole dataset. 
Hence, we see that within this regime our sketch is more accurate without requiring too much more memory, and one to two orders of magnitude faster in terms of the query complexity. 

\begin{remark}
Both sketches need an assumption on the query neighbourhood, and the two are
not comparable. Compressed-sensing recovery has error scaling with the total
mass of $s(q)$, so RACE-CMS is accurate when few items are related to the query
which is the case on its set and graph-similarity benchmarks. On image embeddings
most items are loosely related, the mass stays large at any sketch size, and
recall plateaus. We never recover $s(q)$: Poisson occupancy constrains the
local density without requiring the neighbourhood to be sparse, and candidates
come from LSH probing, so additional memory enlarges the candidate set and
recall keeps improving. 
\end{remark}

The median difference plots \footnote{Median difference is the median value of the difference in the respective metric (approximate recall/accuracy) as we vary compression rates. So a positive median difference corresponds to our scheme consistently out-performing the baseline.} in Figure \ref{fig:sub3-sann} and Figure \ref{fig:sub3-fm} show that \textsf{S-ANN} outperforms \textsf{JL} on both metrics beyond certain values of $\epsilon$. We show a more detailed comparison of approximate recall/accuracy vs compression rate for both datasets in Figure \ref{fig:exp1-main}. Figure \ref{fig:memory-n} shows us that for a fixed choice of $\epsilon$, we can obtain sub-linear sketches for $\eta \geq 0.5$. Putting these experiments together, we can see that it is possible to attain sublinear sketches for Problem \ref{prob:ann} for an appropriate choice of $\epsilon$ with good performance. 

Figure \ref{fig:qps} compares recall and query throughput for the \textsf{JL} baseline (left column) and our \textsf{S-ANN} scheme (right column) across three datasets (rows). For a fixed workload, the figure shows that increasing \textsf{JL}'s projection dimension $k$, steadily improves recall but offers little to no throughput benefit. Similarly for \textsf{S-ANN}, lowering the sampling rate $\eta$ improves recall but doesn't convincingly benefit the query throughput, just as one would expect. However the throughput for our scheme is significantly superior to that for the baseline for all datasets over all choice of parameters for this fixed workload.

\begin{figure}[ht]
\centering
\includegraphics[width=0.4\linewidth]{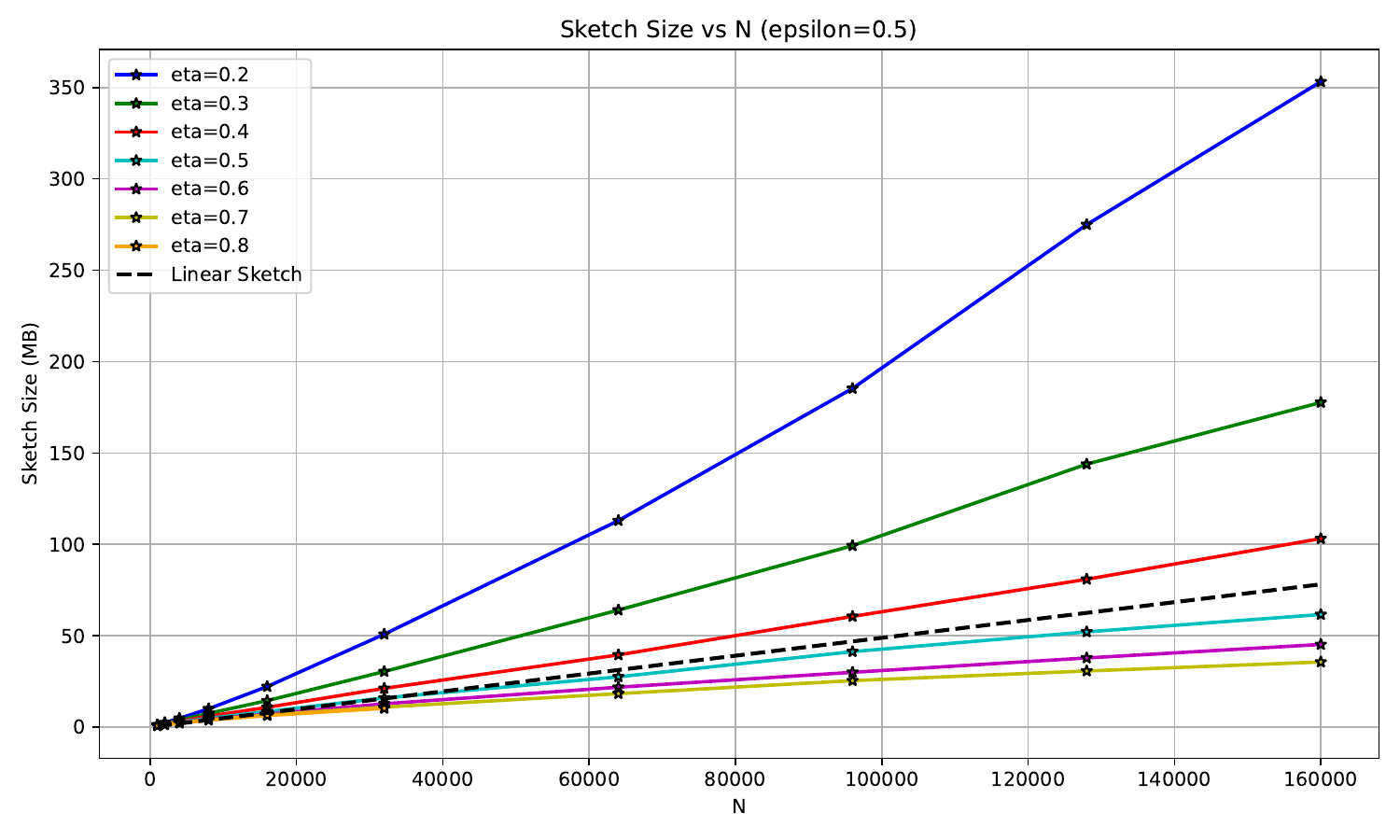}
\caption{Memory requirements scale with stream size $N$ for fixed $\epsilon=0.5$ for the \texttt{sift1m} dataset.}
\label{fig:memory-n}
\end{figure}
\begin{figure}[H]
\centering
\begin{subfigure}{.48\textwidth}
  \centering
  \includegraphics[height=0.14\paperheight,width=.85\linewidth]{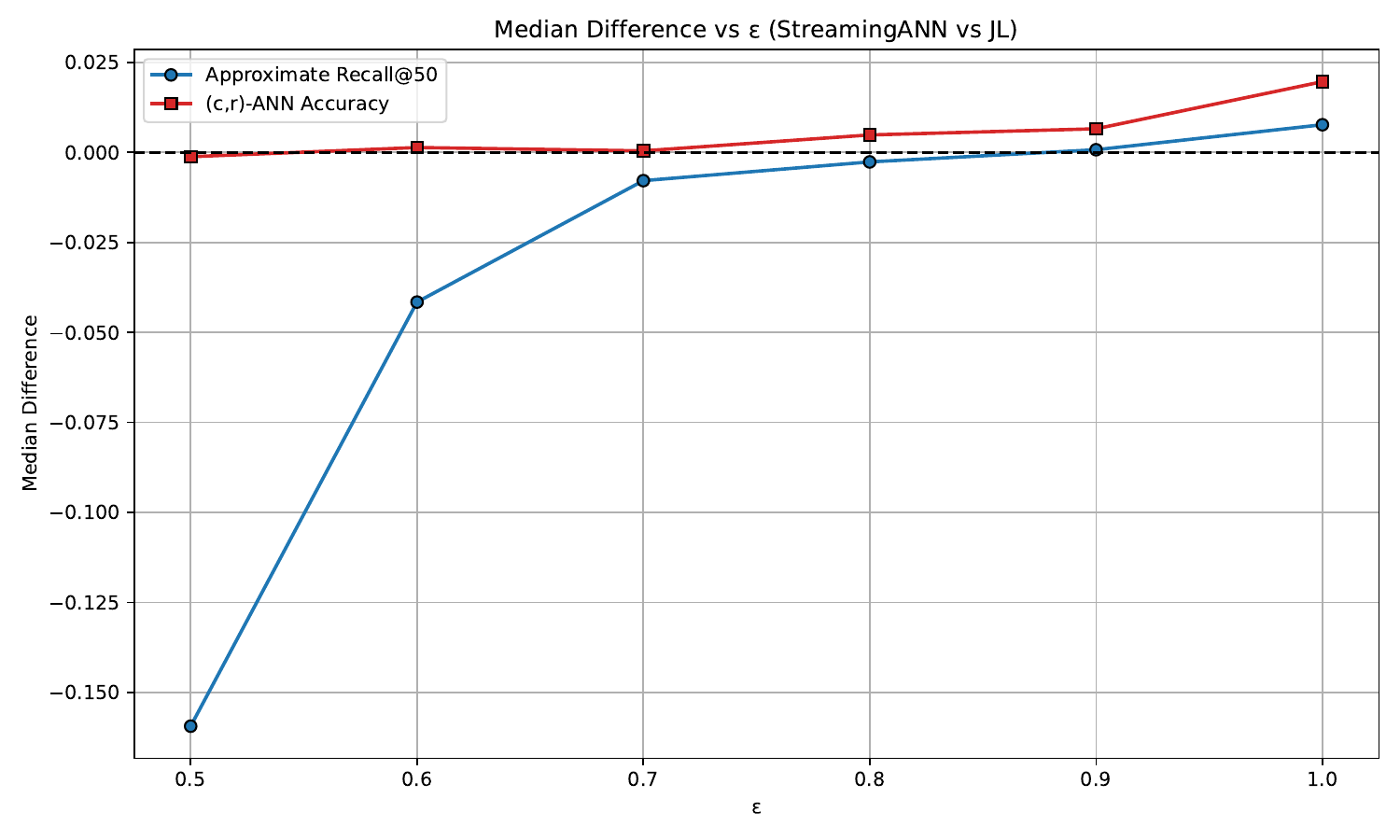}
  \caption{}
  \label{fig:sub3-sann}
\end{subfigure}%
\hfill
\begin{subfigure}{.48\textwidth}
  \centering
  \includegraphics[height=0.14\paperheight,width=.85\linewidth]{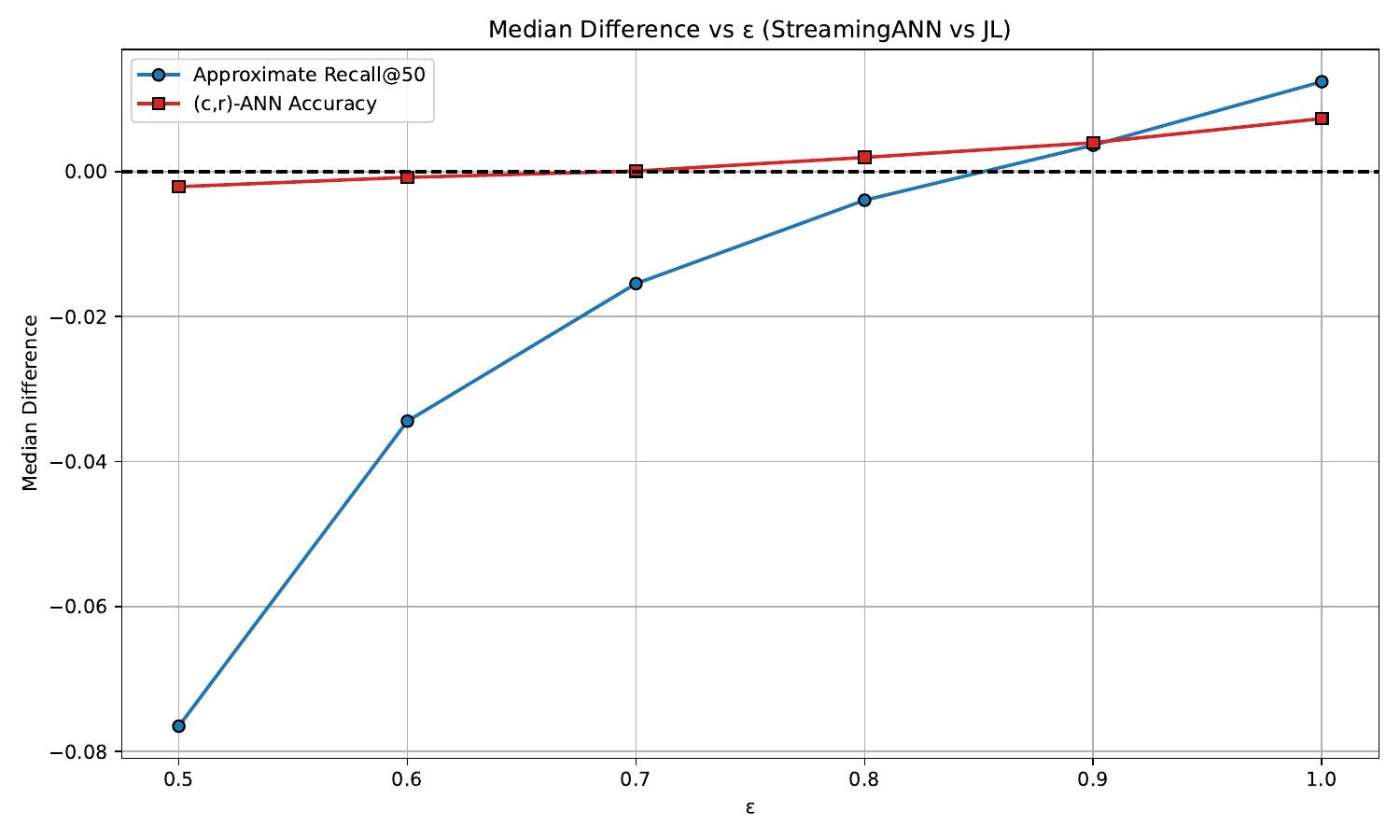}
  \caption{}
  \label{fig:sub3-fm}
\end{subfigure}

\caption{Median differences in Approximate Recall and $(c,r)$-\textsf{ANN} Accuracy over varying $\epsilon$ for \texttt{sift1m} (left) and \texttt{fashion-mnist} (right).}
\label{fig:exp1-summary}
\end{figure}

\begin{figure}[ht]
\centering

% --- Top row ---
\begin{subfigure}{.48\textwidth}
  \centering
  \includegraphics[width=.95\linewidth]{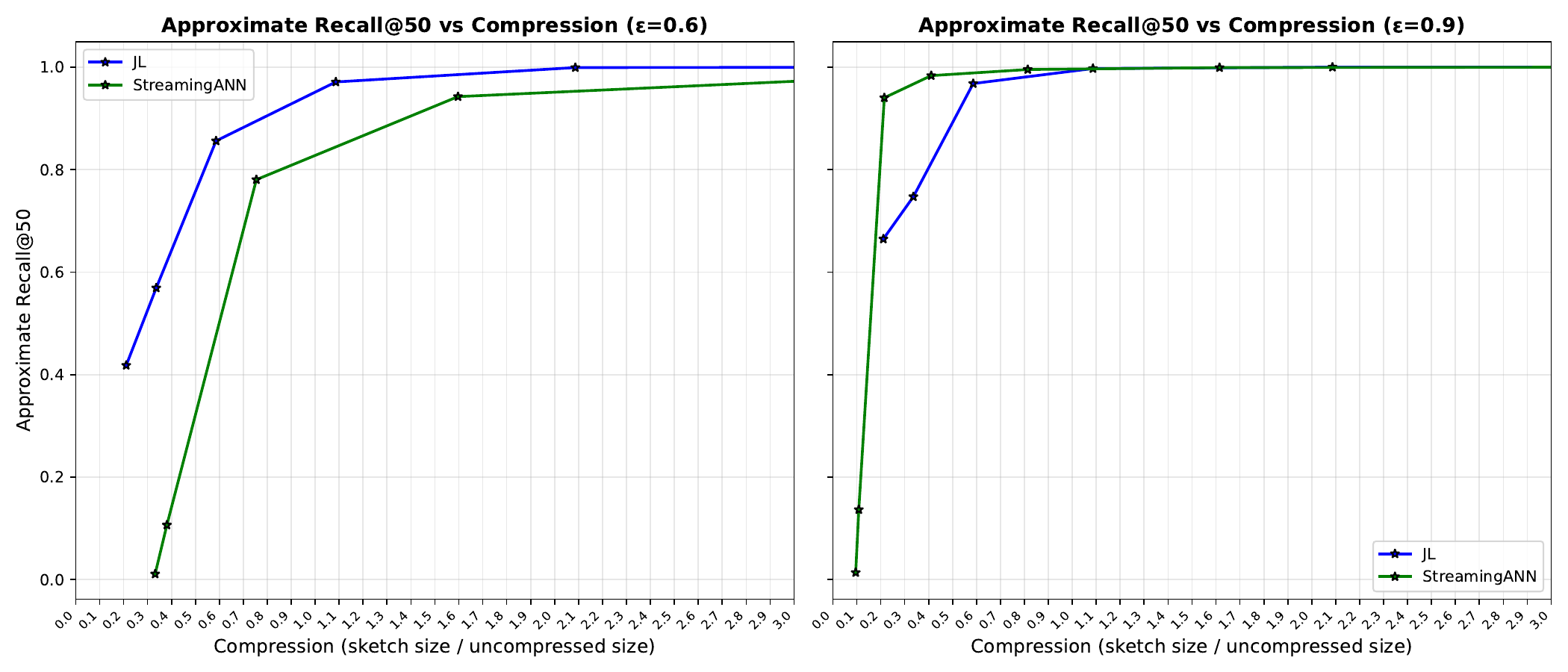}
  \caption{}
  \label{fig:sub1-sann}
\end{subfigure}%
\hfill
\begin{subfigure}{.48\textwidth}
  \centering
  \includegraphics[width=.95\linewidth]{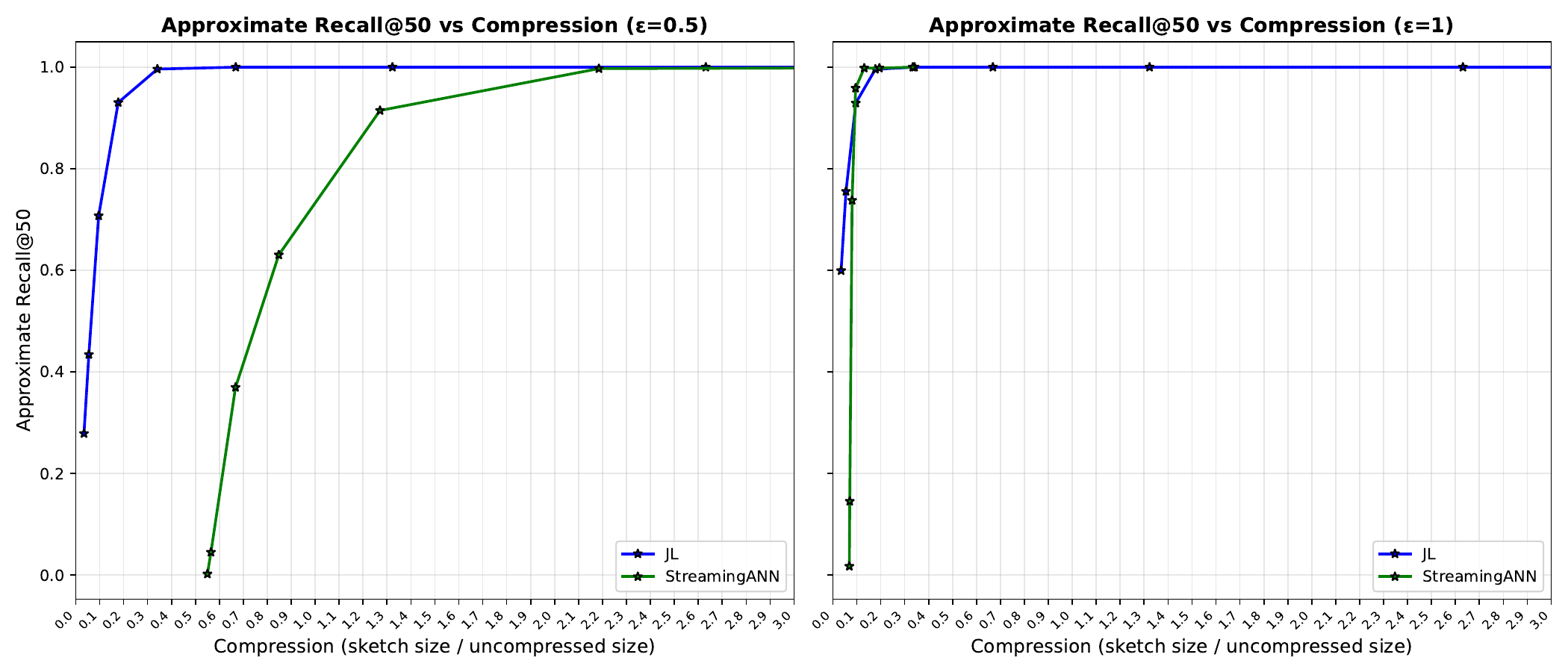}
  \caption{}
  \label{fig:sub1-fm}
\end{subfigure}

% --- Bottom row ---
\begin{subfigure}{.48\textwidth}
  \centering
  \includegraphics[width=.95\linewidth]{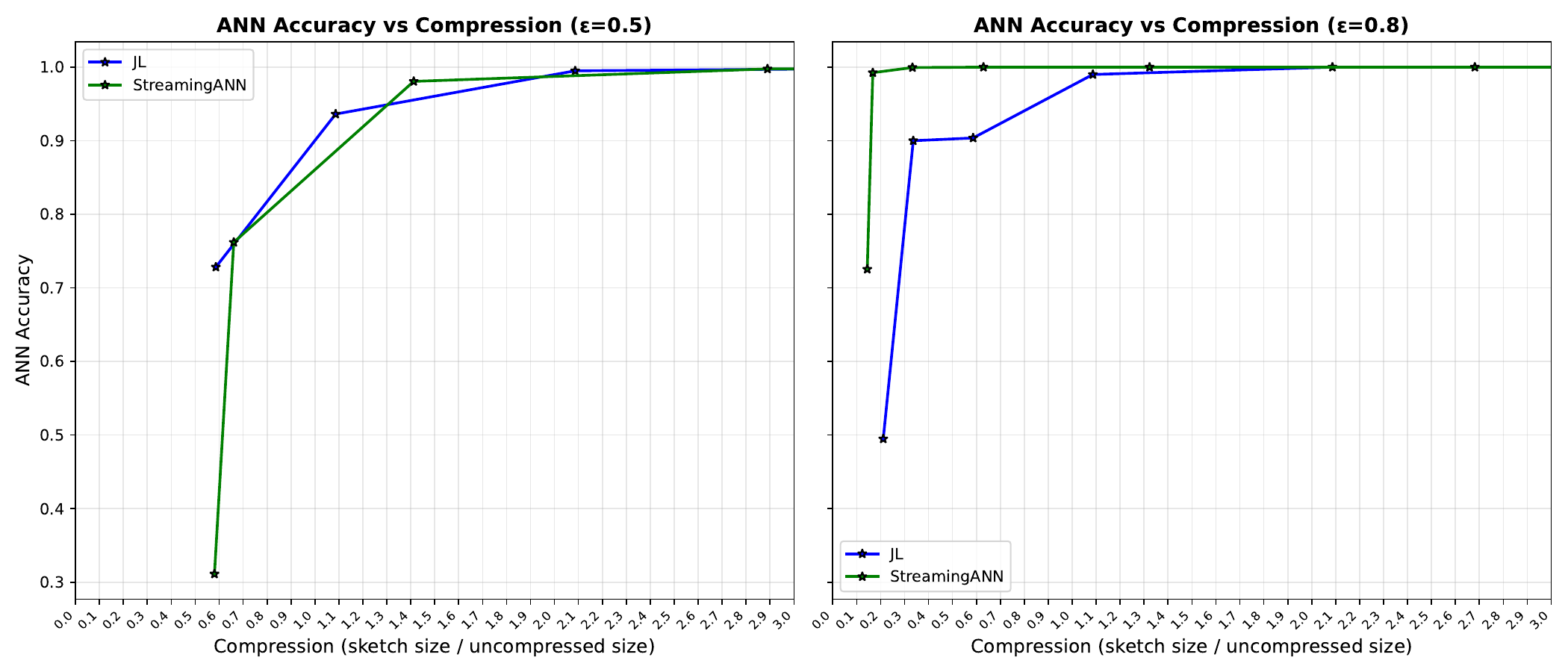}
  \caption{}
  \label{fig:sub2-sann}
\end{subfigure}%
\hfill
\begin{subfigure}{.48\textwidth}
  \centering
  \includegraphics[width=.95\linewidth]{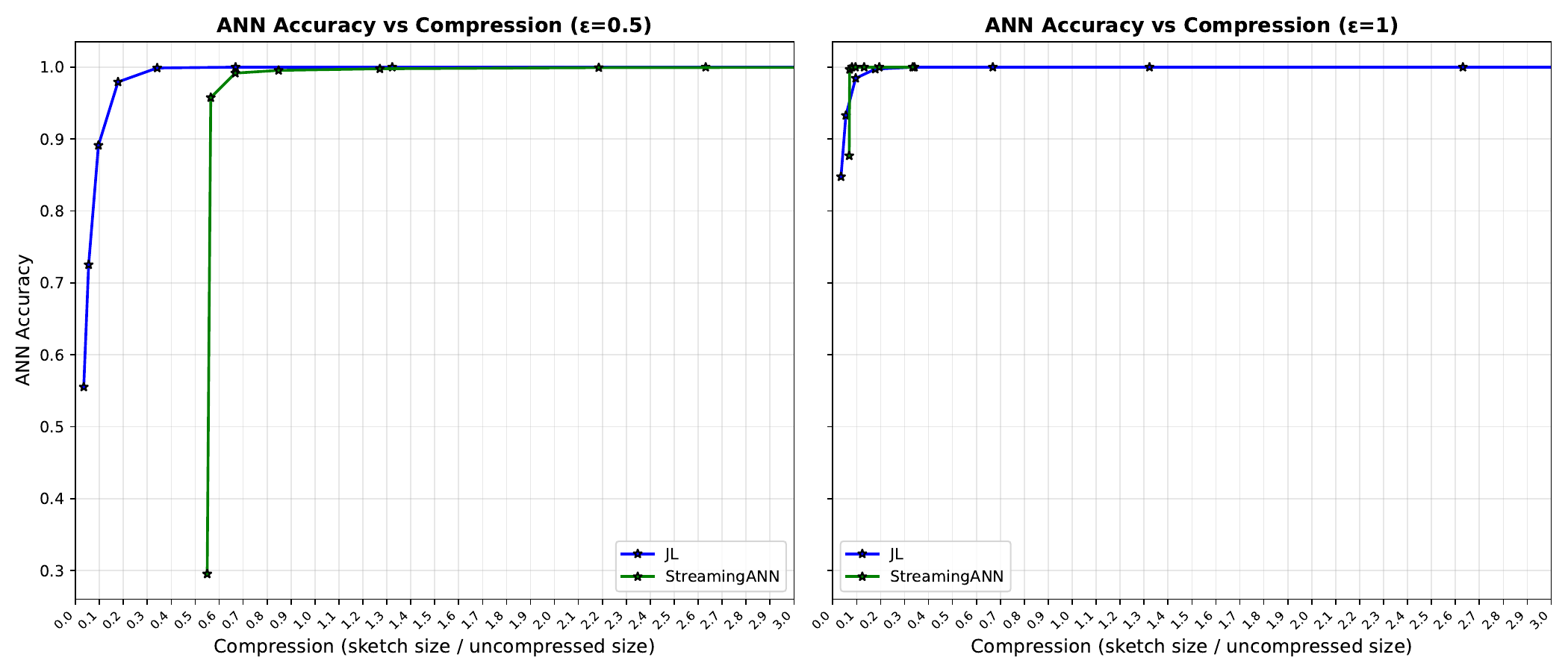}
  \caption{}
  \label{fig:sub2-fm}
\end{subfigure}

\caption{Effect of $\epsilon$ on retrieval performance: Left column shows results on \texttt{sift1m}, Right column shows results on \texttt{fashion-mnist}. (a) and (b) show Recall vs. Compression Rate for two values of $\epsilon$, and (c) and (d) show $(c,r)$-\textsf{ANN} Accuracy vs. Compression Rate for the same settings.}
\label{fig:exp1-main}
\end{figure}

\begin{figure}[H]
\centering
\includegraphics[width=0.9\linewidth]{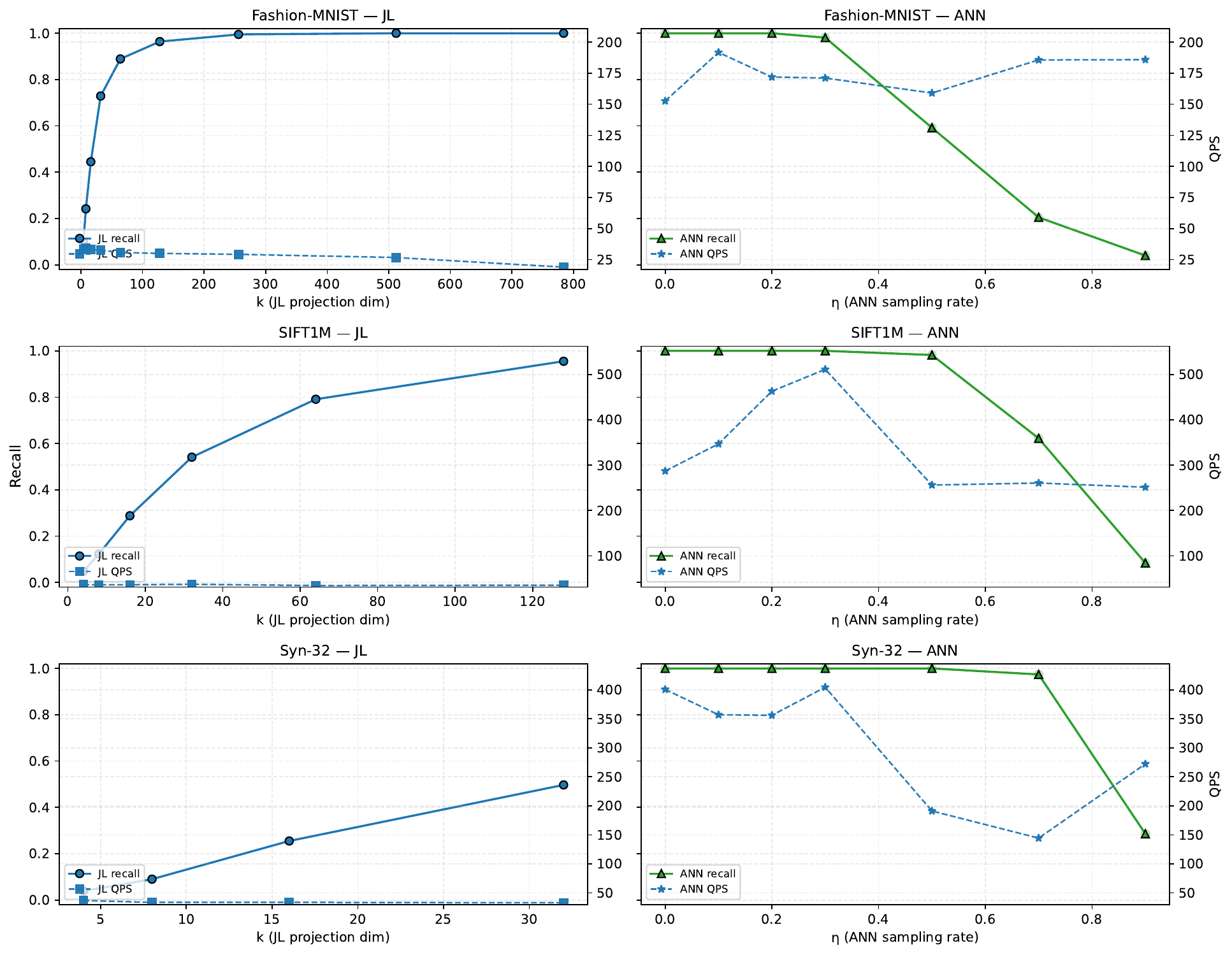}
\caption{Query throughput (\textsf{QPS}) and Recall variation for both \textsf{JL} (left) and \textsf{S-ANN} (right) across datasets. 
Each row corresponds to a dataset (\texttt{fashion-mnist}, \texttt{sift1m}, and \texttt{syn-32}), showing Recall (solid) and Queries-per-Second (dashed, right axis) as parameters $k$ and $\eta$ vary.}
\label{fig:qps}
\end{figure}

\begin{figure}[ht]
\centering
\includegraphics[width=0.7\linewidth]{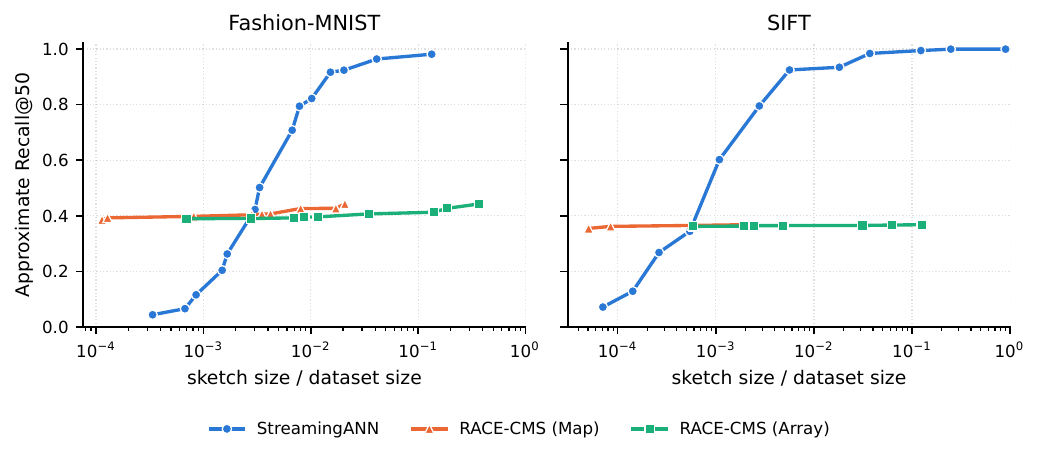}
\caption{Comparison of Recall vs. sketch size between \textsf{S-ANN} and RACE-CMS across datasets}
\label{fig:memory_recall_cms}
\end{figure}

\subsection{Experiments for \textsf{A-KDE}} 
We investigate the impact of various parameters, including sketch size and sliding window size, on the accuracy of our algorithm for different datasets. We also compare \textsf{SW-AKDE} with \textsc{RACE}\cite{coleman2020sub}.

\emph{Datasets.}
We conduct experiments on both synthetic and real-world datasets. For the synthetic dataset, we generate $10{,}000$ $200$-dimensional vectors sampled from $10$ different multivariate gaussian distributions. For every $1000$ datapoints, we use a specific gaussian distribution. We generate this data set 50 times at random and conduct \emph{Monte Carlo} simulations to study the performance of our algorithm. The real-world datasets used are: (a) News Headlines~\cite{Kulkarni2017MillionHeadlines} — It consists of news headlines published over eighteen years by a reputable news agency. We select approximately $80{,}000$ headlines and convert them into $384$-dimensional text embeddings using the sentence transformer \textbf{all-MiniLM-L6-v2}\footnote{\href{https://www.kaggle.com/code/masterchen09/all-minilm-l6-v2}{Kaggle link}}. (b) ROSIS Hyperspectral Data~\cite{Sowmya2019} — It contains hyperspectral images captured by the \emph{ROSIS} sensor. Each image pixel is treated as a high-dimensional data vector (of dimension $103$ in this case) corresponding to a frequency spectrum. We use around $42{,}000$ data.

\emph{Implementation.} \textsf{SW-AKDE} is implemented using two different \textsc{LSH} kernels:
(a) Angular \textsc{LSH}, and
(b) $p$-stable Euclidean \textsc{LSH}. To bound the range of the $p$-stable \textsc{LSH} functions, we employ rehashing. We have used Python to implement all the algorithms. The bandwidth parameter, denoted by $p$ in Algorithm~\ref{KDE_algo}, was set to $1$ for all experiments. All results reported here correspond to the single-query setting. For the \textsf{SW-AKDE}, the relative error of the \textsc{EH} was taken as $0.1$. Hence, the upper bound on the theoretical relative error for the \textsc{A-KDE} estimate is 0.21 (follows from Lemma \ref{kdeestimate}). It is observed that the experimental errors are much less than this theoretical value, even for a small number of rows of the sketch.

\emph{Baseline.} We have used the \textsc{RACE} algorithm proposed in \cite{coleman2020sub} for comparison. This algorithm works for general streaming settings, whereas our algorithm works in a sliding window model.

\emph{Experimental Setup.} We have performed the following experiments to show the effect of various parameters on the performance of our algorithm.
\begin{enumerate}
    \item \textbf{Sketch Size vs. Error:} We show the variation of the log of the mean relative error with sketch size(Fig.~\ref{fig:exp:main}). We fixed the \textsc{EH} relative error as $\epsilon=0.1$ and the window size as 450. We ran the experiment for $10{,}000$ streaming data and measured the mean error on $1000$ query data. We did separate experiments on both real-world and synthetic datasets using Euclidean hash and angular hash. We varied the row size exponentially as: $100,200,400,800,1600,3200$.    

    \item \textbf{Window Size Effect:} We demonstrate the effect of varying the sliding window size on the mean relative error for the real-world datasets(Fig.~\ref{fig:kde_exp}). We have taken the window size as: $64,128,256,512,1024,2048$. For each window size, we have plotted the log of the mean relative error against the number of rows. The row sizes were varied similarly to those of Experiment 1. 

    \item \textbf{Comparison with RACE:} We compared \textsf{SW-AKDE} algorithm using angular hash and a window size of $260$ with \textsc{RACE} (Fig.~\ref{fig:kdecmp}) for both the real-world and synthetic datasets. We have plotted the log of the mean relative error against the row size, where row size has been taken similarly to Experiment 1.

\end{enumerate}
\paragraph{Discussion:} 

For \textsf{SW-AKDE}, the mean error reduces with an increase in sketch size for Euclidean kernels(Fig.~\ref{fig:exp:main1}). In case of angular hash, the mean error decreases in general for text data but is erratic for image data(Fig.~\ref{fig:exp:main2}). For synthetic data, both kernels show a regular decrease with increasing sketch size(Figs.~\ref{fig:exp:main3},~\ref{fig:exp:main4}). It is observed that higher window sizes minimize error for text data(Fig.~\ref{fig:kde_exp1}) while for image data the mean relative error is minimized for window size=$256$(Fig.~\ref{fig:kde_exp2}). When compared to \textsc{RACE}, our algorithm gives similar performance(Fig.~\ref{fig:kdecmp}) for both real-world and synthetic datasets. These results validate our theoretical guarantees and demonstrate practical effectiveness across tasks.
\begin{figure}[ht]
\begin{subfigure}[b]{.5\textwidth}
  \centering
  \includegraphics[trim={1cm 0cm 1.5cm 0cm},clip,width=.95\linewidth]{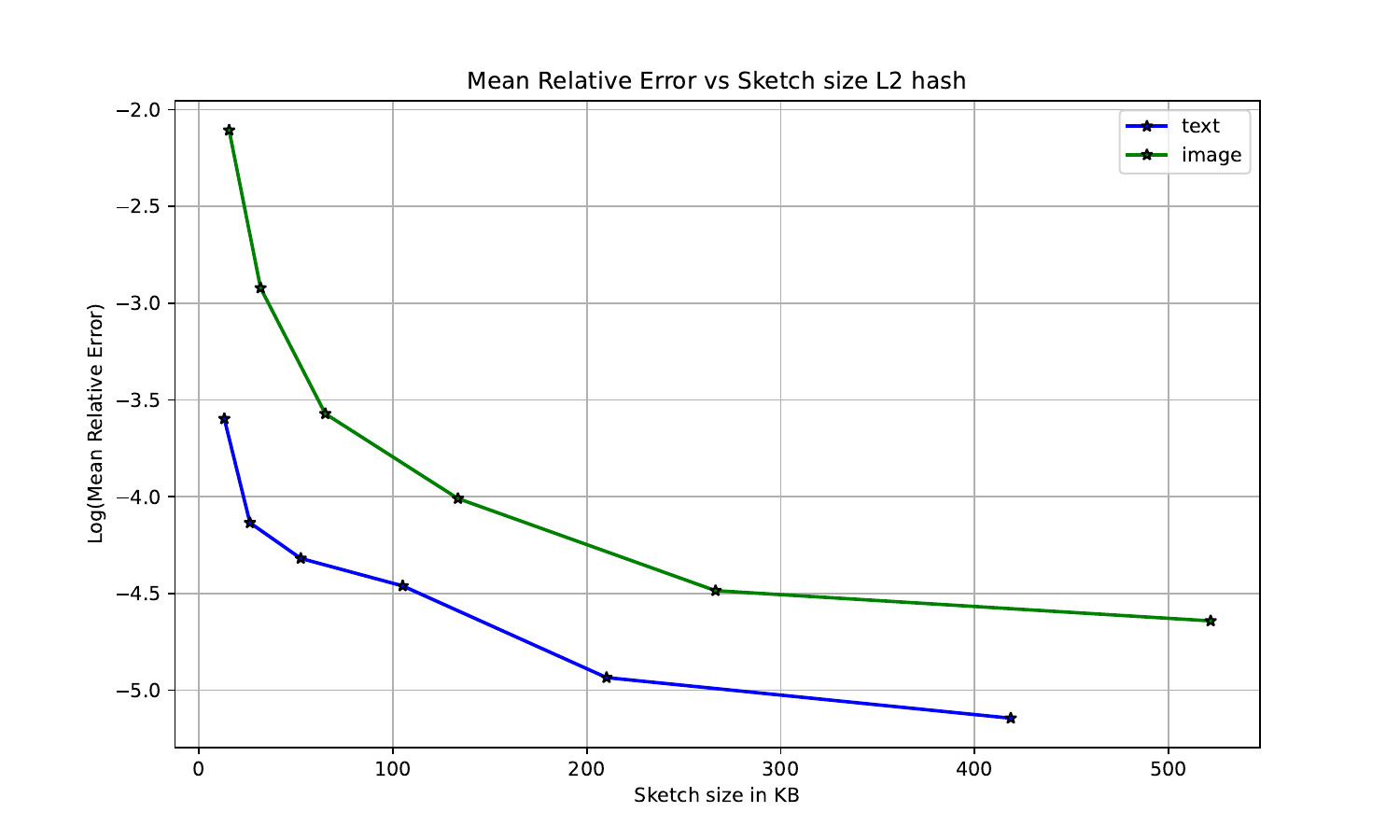}
  \caption{}
  \label{fig:exp:main1}
\end{subfigure}%
\hspace{0.02\textwidth}
\begin{subfigure}[b]{.5\textwidth}
  \centering
  \includegraphics[trim={1cm 0cm 1.5cm 0cm},clip,width=.95\linewidth]{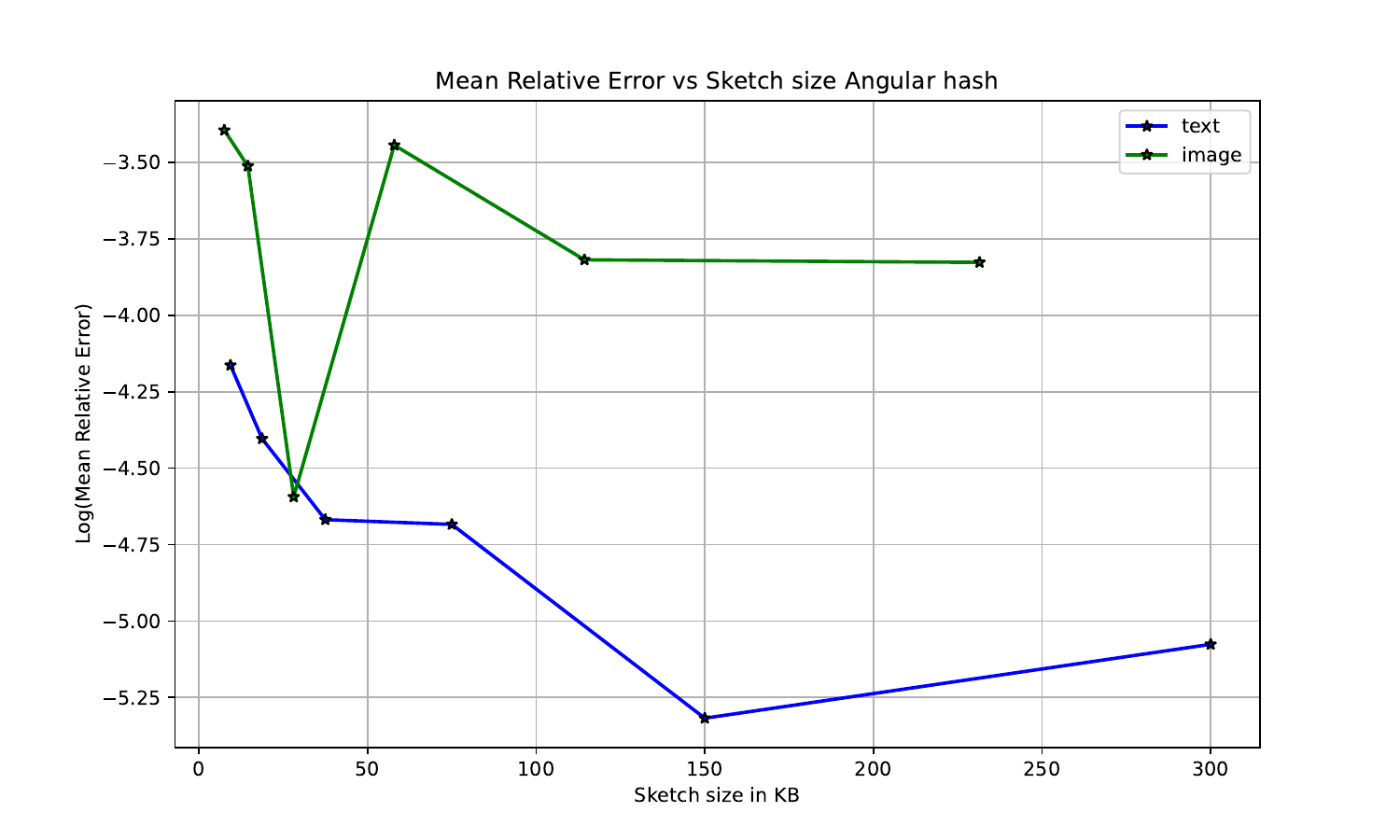}
  \caption{}
  \label{fig:exp:main2}
\end{subfigure}
\begin{subfigure}[b]{.5\textwidth}
  \centering
  \includegraphics[trim={1cm 0cm 1.5cm 0cm},clip,width=.95\linewidth]{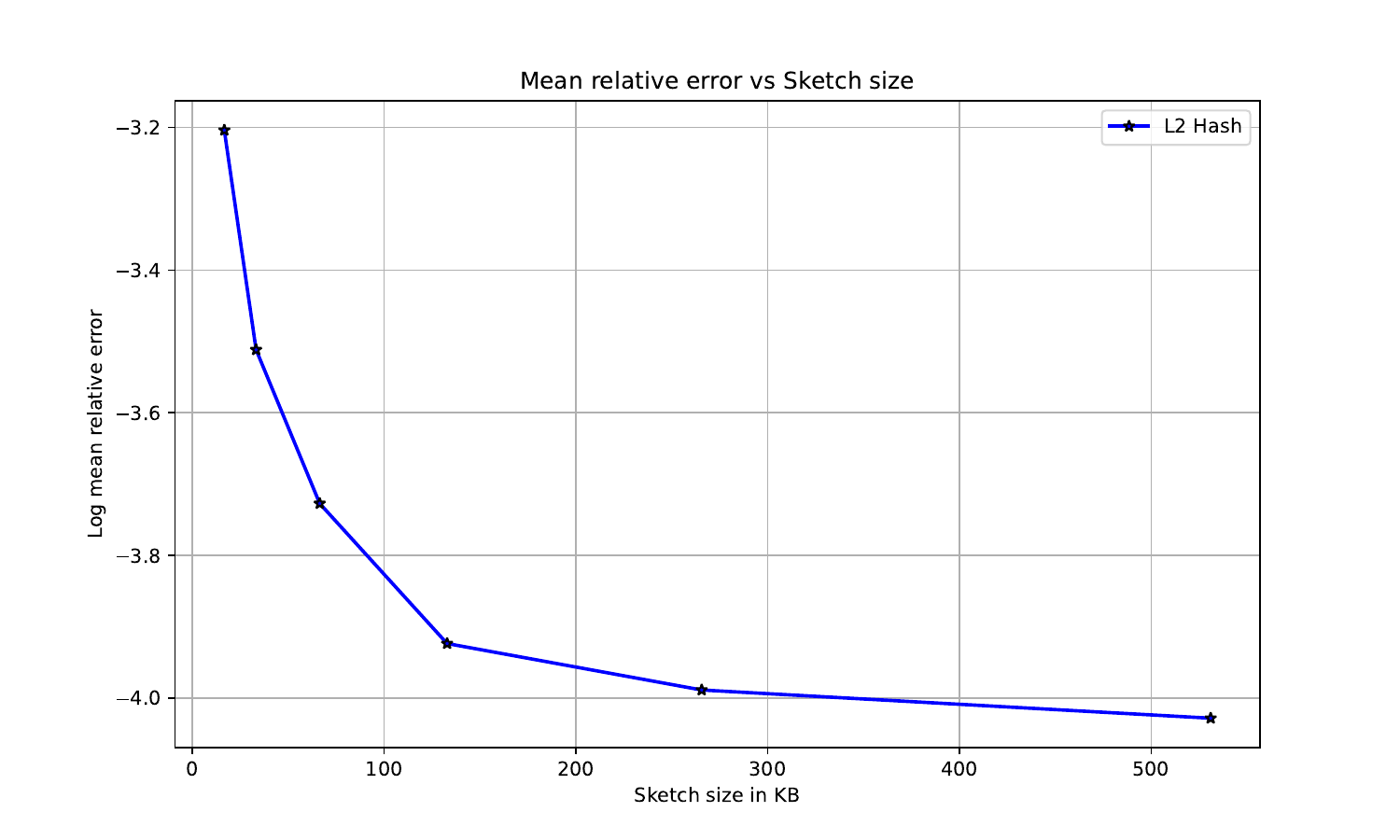}
  \caption{}
  \label{fig:exp:main4}
\end{subfigure}
\hspace{0.02\textwidth}
\begin{subfigure}[b]{.5\textwidth}
  \centering
  \includegraphics[trim={1cm 0cm 1.5cm 0cm},clip,width=.95\linewidth]{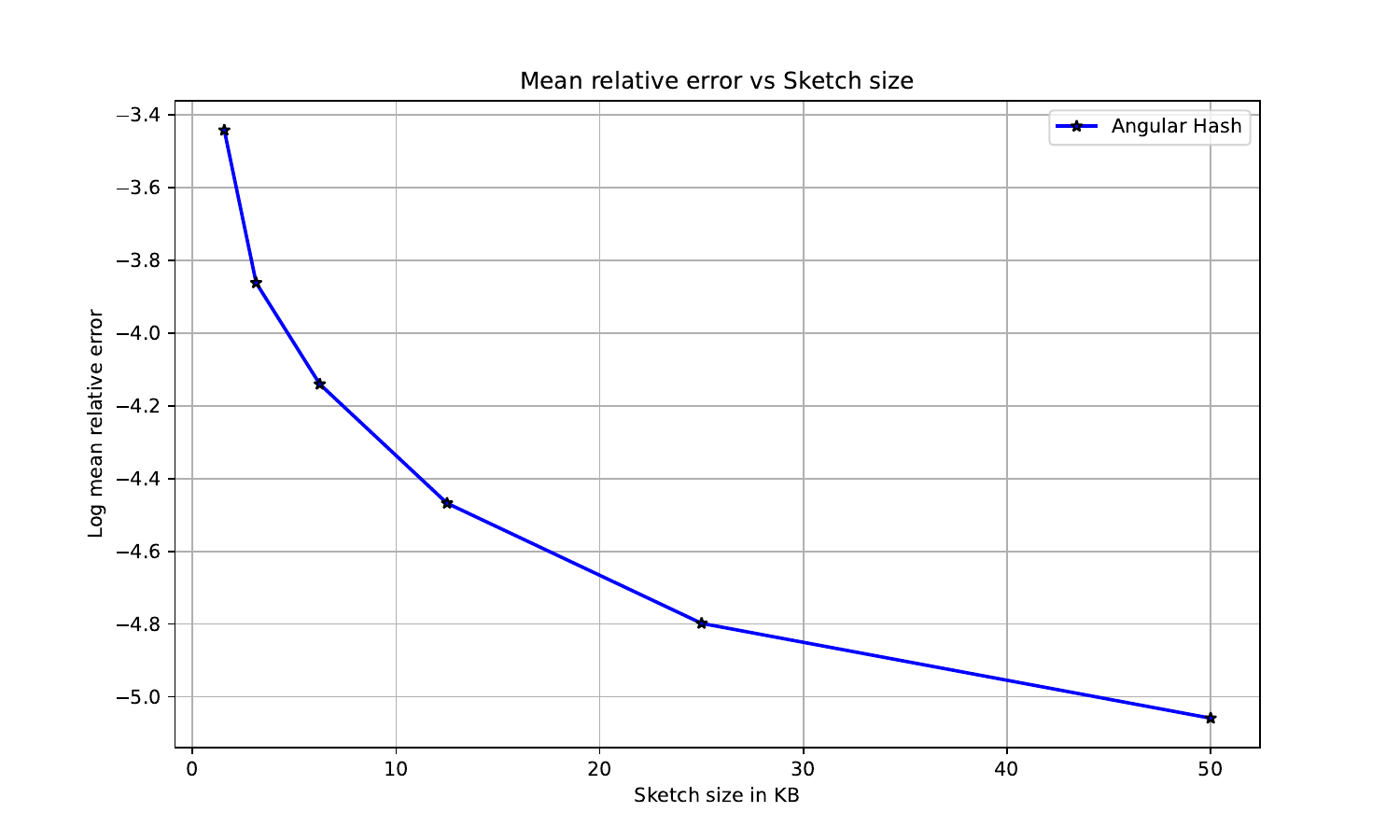}
  \caption{}
  \label{fig:exp:main3}
\end{subfigure}%

\caption{Effect of sketch size on \textsf{KDE} estimates (a) Real-world data with $p$-stable hash (b) Real-world data with Angular hash (c) Synthetic data with $p$-stable hash (d) Synthetic data with Angular hash}
\label{fig:exp:main}
\end{figure}

\begin{figure}[H]
\centering
\begin{subfigure}[b]{.5\textwidth}
  \centering
  \includegraphics[width=.95\linewidth]{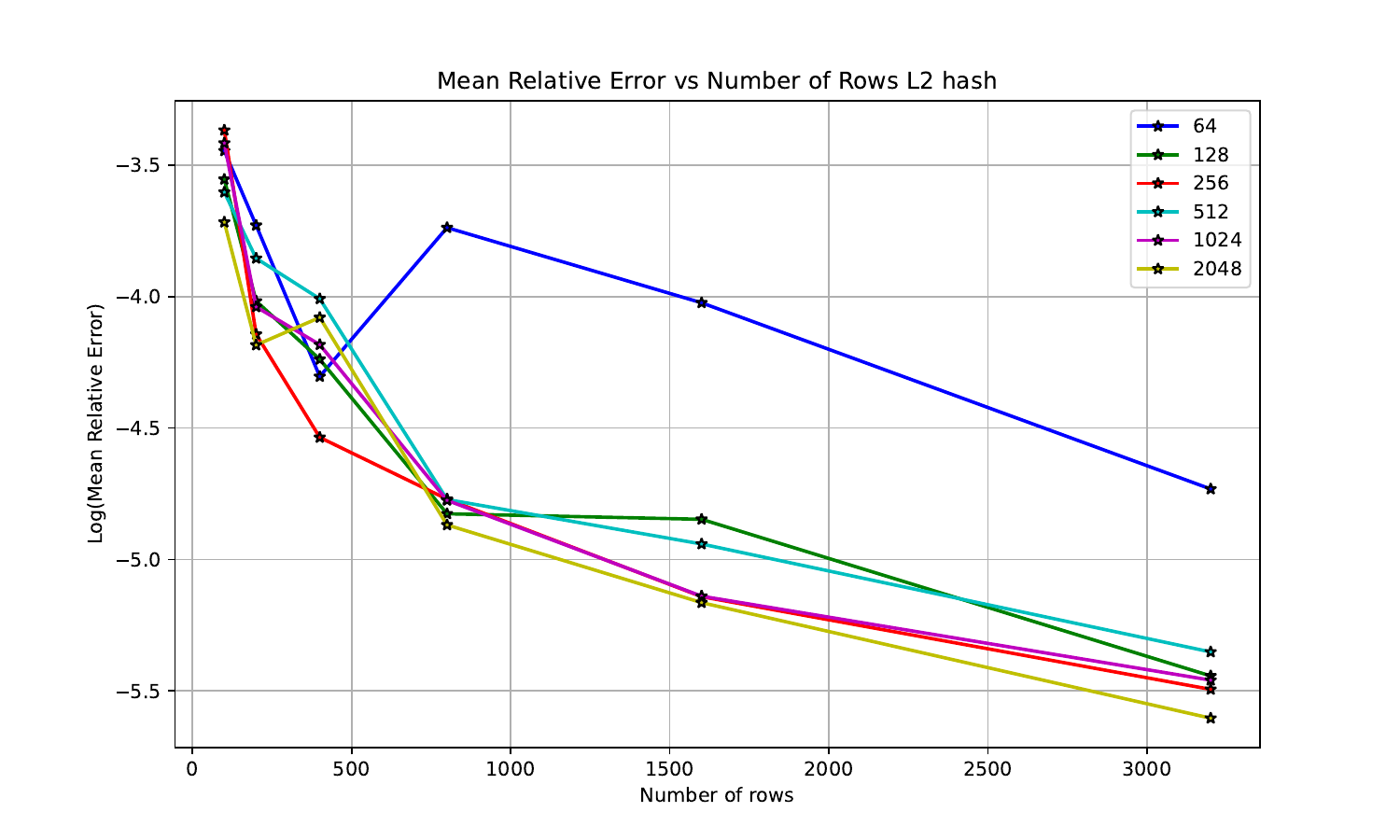}
  \caption{}
  \label{fig:kde_exp1}
\end{subfigure}%
\begin{subfigure}[b]{.5\textwidth}
  \centering
  \includegraphics[width=.95\linewidth]{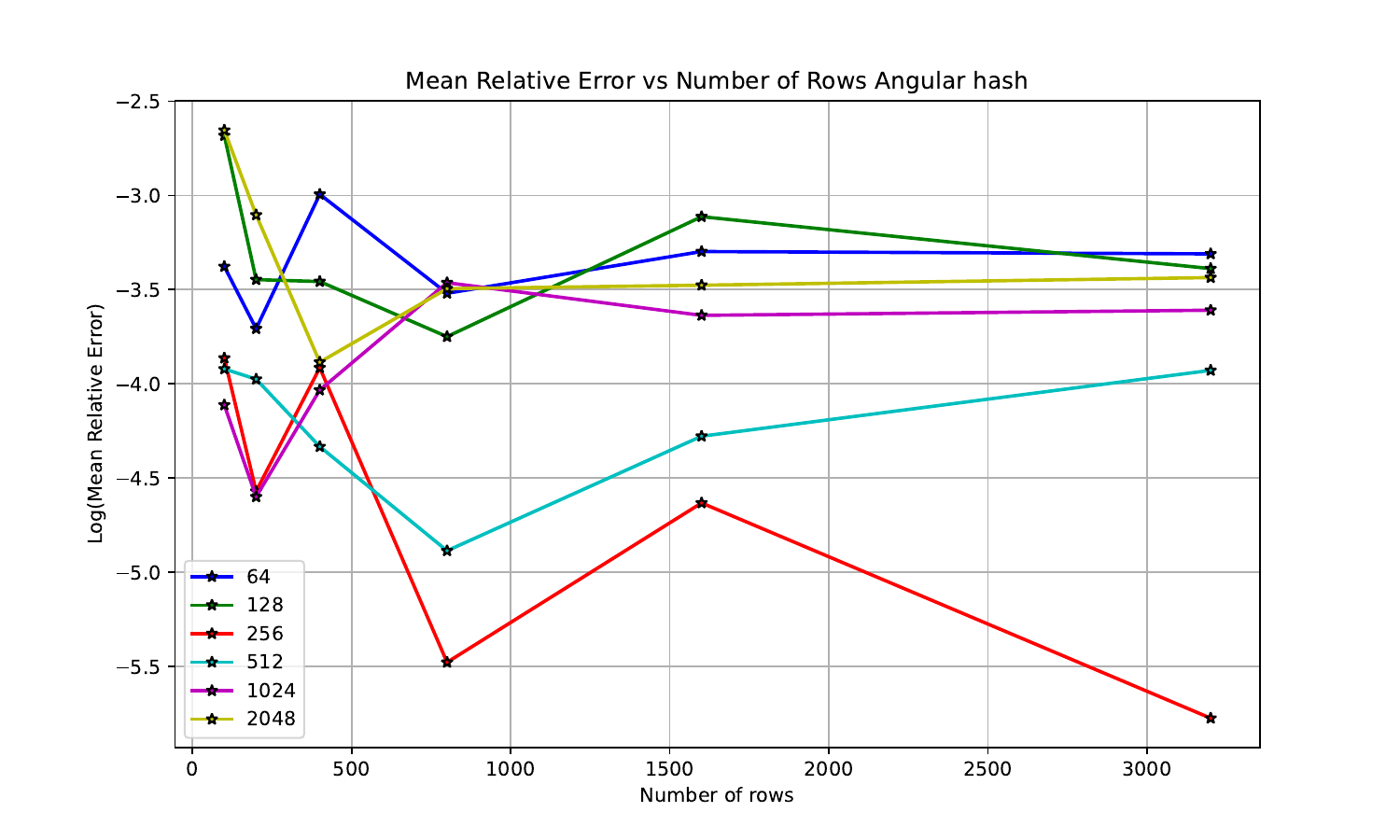}
  \caption{}
  \label{fig:kde_exp2}
\end{subfigure}
\caption{Effect of the window size on mean relative error for \textsf{SW-AKDE} with (a) Euclidean hash on news headlines data (b) angular hash on ROSIS image data.}
\label{fig:kde_exp}
\end{figure}

\begin{figure}[H]
\centering
\begin{subfigure}{.3\textwidth}
  \centering
  \includegraphics[trim={1cm 0cm 1cm 0cm},clip,width=.95\linewidth]{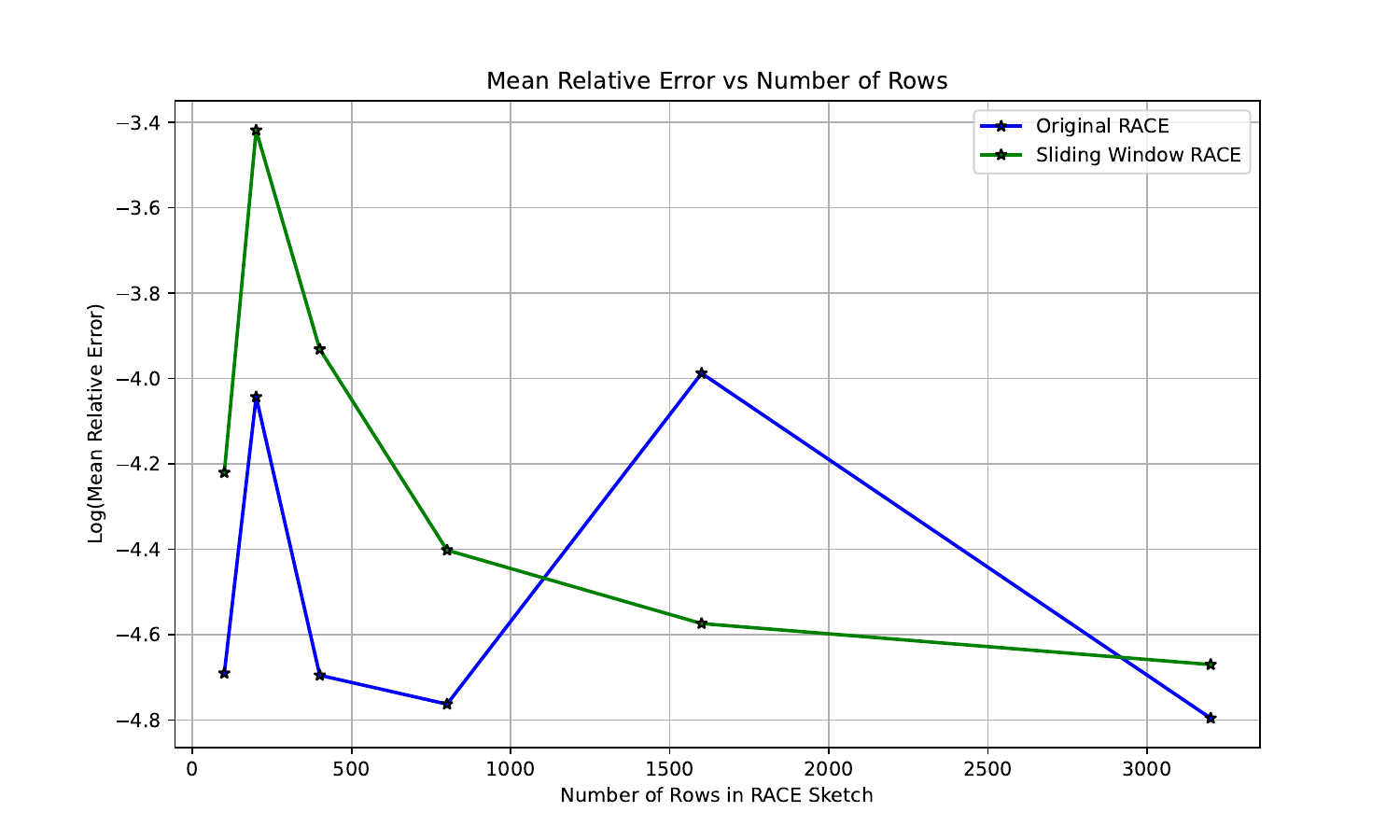}
  \caption{}
  \label{fig:kdecmp1}
\end{subfigure}%
\begin{subfigure}{.3\textwidth}
  \centering
  \includegraphics[trim={1cm 0cm 1cm 0cm},clip,width=.95\linewidth]{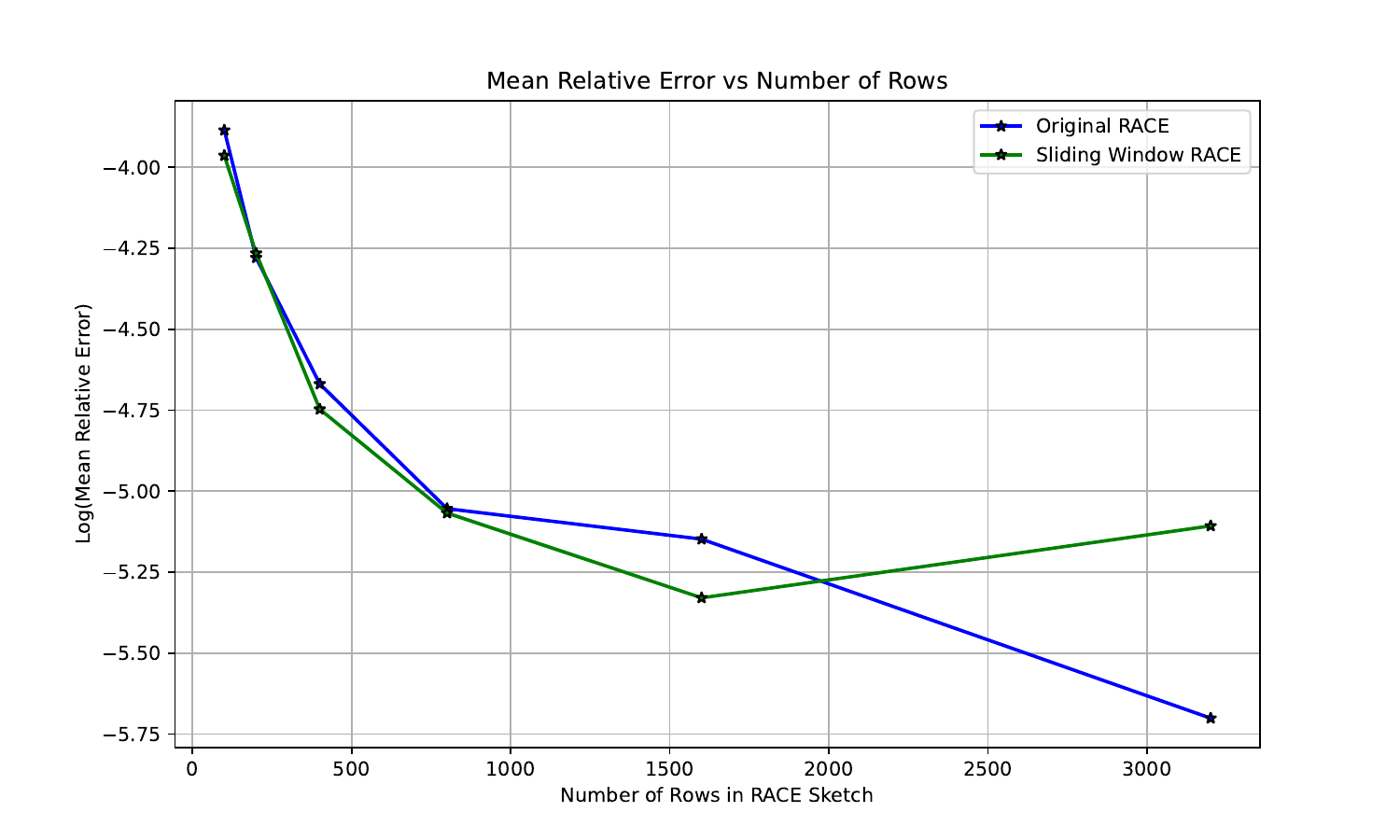}
  \caption{}
  \label{fig:kdecmp2}
\end{subfigure}
\begin{subfigure}{.3\textwidth}
  \centering
  \includegraphics[trim={1cm 0cm 1cm 0cm},clip,width=.95\linewidth]{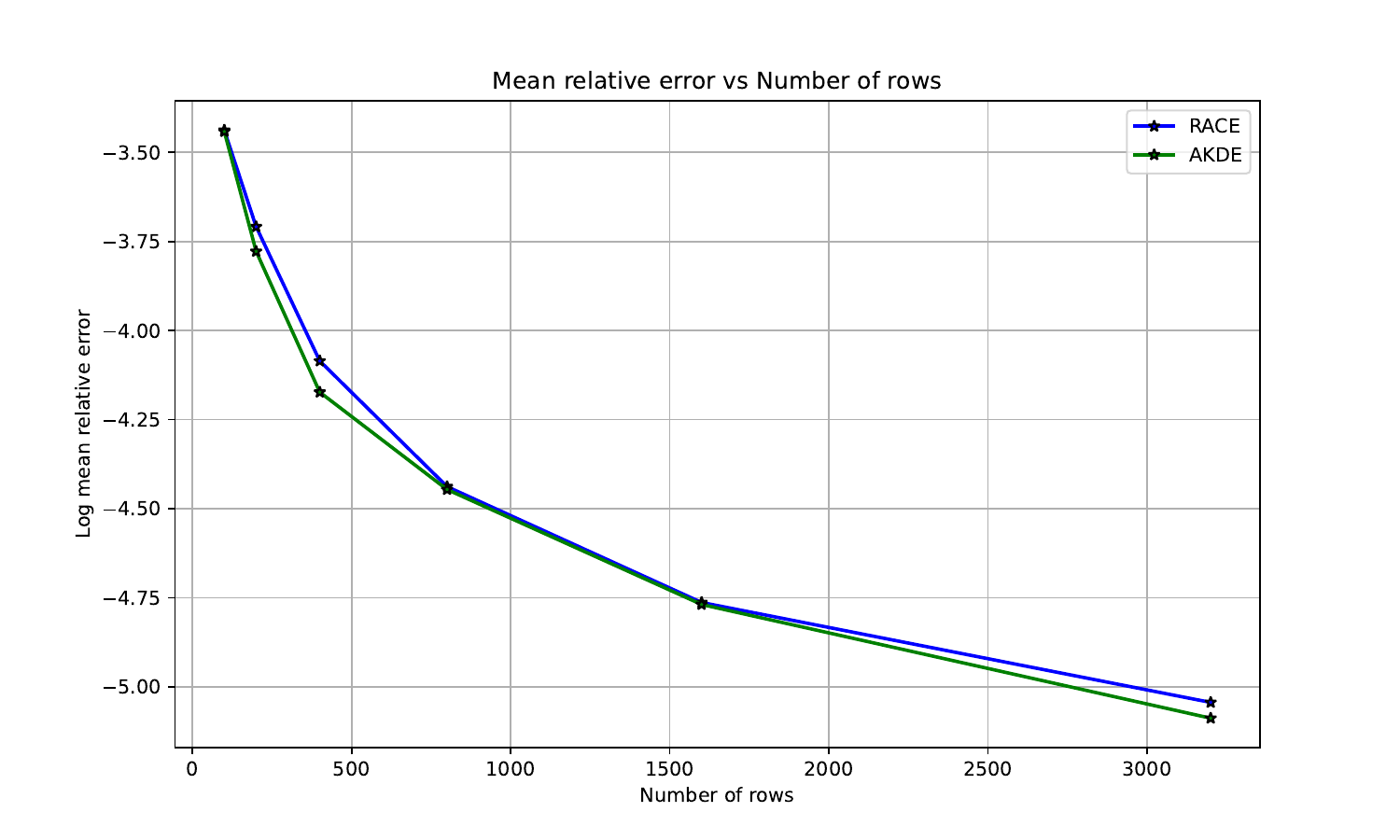}
  \caption{}
  \label{fig:kdecmp3}
\end{subfigure}
\caption{Comparison between \textsc{RACE} structure with Angular Hash and \textsf{SW-AKDE} with Angular Hash (a) \texttt{ROSIS Hyperspectral data}\hspace{5px} (b) \texttt{News headlines} (c) Synthetic data}
\label{fig:kdecmp}
\end{figure}

\section{Conclusion}

We presented new sketching algorithms for two fundamental streaming problems: Approximate Nearest Neighbor (\textsf{ANN}) search and Approximate Kernel Density Estimation (\textsf{A-KDE}). For \textsf{ANN}, we showed that under Poisson-distributed inputs, storing only a sublinear fraction of the stream suffices to preserve $(c,r)$-\textsf{ANN} guarantees, yielding simple, turnstile-robust sketches that support both single and batch queries and outperform the Johnson--Lindenstrauss baseline in accuracy and throughput. 
For \textsf{A-KDE}, we delicately integrated \textsc{RACE} with \textsc{EH} to design the first sketch that works in the general sliding-window model, providing provable error guarantees and empirical performance comparable to \textsc{RACE} while efficiently handling data expiration.

Several open directions arise from this work. For \textsf{ANN}, it would be valuable to understand how the guarantees extend beyond the Poisson model to other distributional settings, and to characterize finer trade-offs between space, time, and approximation quality. Another promising direction is to explore data-aware variants of \textsc{LSH}~\cite{AndoniR15, AndoniNNRW18} that adapt sketch parameters to the underlying distribution. 
For \textsf{A-KDE}, our experiments (Fig.~\ref{fig:kde_exp}) show that the sliding-window size significantly influences performance, raising the natural question of how to select this parameter optimally—potentially as a function of the relative error of the \textsc{EH}~\cite{datar2002maintaining}, the sketch width, or the observed data dynamics. Developing adaptive mechanisms for adjusting the window size based on the evolving data distribution remains an intriguing direction for future work.

%\begin{itemize}
 %   \item Quickly summarize all what we have done. We have proposed an approximate algorithm for estimating \textsc{KDE} in streaming data in the sliding window model. It works in both single query and batch query setting. We have combined \textsc{RACE}\cite{coleman2020sub} and \textsc{EH}\cite{datar2002maintaining} to design a sketch which supports \textsc{KDE} queries in the sliding window model. We give theoretical bounds on the relative error of our algorithm with respect to the actual \textsc{KDE}. We have conducted extensive experiments on both real-world and synthetic datasets to support our claims. Our algorithm also performs comparatively to \textsc{RACE}\cite{coleman2020sub}.
   % \item Future directions - \textsf{ANN} what happens for other distributions; What are the other trade-offs in terms of time and space... What can we say about data aware \textsc{LSH}. 
   % \item \textsf{A-KDE} - We have seen that the sliding window size affects the performance of \textsf{SW-AKDE}(Fig.~\ref{fig:kde_exp}). It leads to an interesting question on how to choose the optimal window size and whether it is dependent on other parameters like the relative error of the \textsc{EH}, row size of the sketch etc. We can also investigate if the window size can be set adaptively based on the incoming data distribution.
%\end{itemize}

\section{Acknowledgment}
The third author acknowledges David Woodruff for helpful pointers and comments regarding streaming models. 
\newcommand{\etalchar}[1]{$^{#1}$}

\end{document}